\documentclass{article} 
\usepackage{iclr2026_conference,times}

\usepackage{hyperref}
\usepackage{url}
\usepackage{amsmath, amssymb, amsthm, mathtools}
\usepackage{thmtools}
\usepackage{enumitem}
\usepackage{booktabs}
\usepackage{graphicx}
\usepackage{caption}
\usepackage{subcaption}
\usepackage{tabularx}
\usepackage{float}

\newtheorem*{theorem*}{Theorem}
\newtheorem*{proposition*}{Proposition}
\newtheorem{theorem}{Theorem}[section]

\newtheorem{proposition}[theorem]{Proposition}

\theoremstyle{definition}
\newtheorem{definition}[theorem]{Definition}

\renewcommand{\S}{Section }

\title{Quantifying Memory Use in Reinforcement Learning with Temporal Range}

\author{Rodney Lafuente-Mercado \\
MIT Lincoln Laboratory \\
\texttt{Rodney.LafuenteMercado@ll.mit.edu} \\
\And
Daniela Rus \\
CSAIL, MIT \\
\texttt{rus@csail.mit.edu} \\
\And
T. Konstantin Rusch \\
ELLIS Institute T\"ubingen \&\\
Max Planck Institute for Intelligent Systems \&\\
T\"ubingen AI Center \\
\texttt{tkrusch@tue.ellis.eu}
}

\iclrfinalcopy
\begin{document}

\maketitle
\begin{abstract}
How much does a trained RL policy actually use its past observations? We propose
\emph{Temporal Range}, a model-agnostic metric that treats first-order
sensitivities of multiple vector outputs across a temporal
window to the input sequence as a temporal influence profile and summarizes it
by the magnitude-weighted average lag. Temporal Range is computed via
reverse-mode automatic differentiation from the Jacobian blocks
$\partial y_s/\partial x_t\in\mathbb{R}^{c\times d}$ averaged
over final timesteps $s\in\{t+1,\dots,T\}$ and is well-characterized in the
linear setting by a small set of natural axioms. Across diagnostic and control
tasks (POPGym; flicker/occlusion; Copy-$k$) and architectures (MLPs, RNNs,
SSMs), Temporal Range (i) remains small in fully observed control, (ii) scales
with the task's ground-truth lag in Copy-$k$, and (iii) aligns with the minimum
history window required for near-optimal return as confirmed by window
ablations. We also report Temporal Range for a compact Long Expressive Memory
(LEM) policy trained on the task, using it as a proxy readout of task-level
memory. Our axiomatic treatment draws on recent work on range measures,
specialized here to temporal lag and extended to vector-valued outputs in the RL
setting. Temporal Range thus offers a practical per-sequence readout of memory
dependence for comparing agents and environments and for selecting the shortest
sufficient context.
\end{abstract}

\section{Introduction}
Reinforcement learning (RL) has a long-standing history of utilizing memory to
improve performance in complex environments
\citep{hausknecht2015deep,berner2019dota,chen2021decision,lu_structured_2023}.
Examples of machine learning models that incorporate memory include classical
Recurrent Neural Networks (RNNs) such as Long Short-Term Memory (LSTM) models
\citep{hochreiter1997long}, Transformer \citep{vaswani2017attention}, and
recently State-Space Models (SSMs) \citep{gu2021efficiently,gu2023mamba}.
However, a rigorous analysis and quantitative measure of the extent to which a
\emph{trained} policy utilizes historical information remains largely absent.
This matters in partially observed settings: if effective history dependence is
short, simpler architectures or shorter contexts suffice; if it is long, we
should see it directly in the learned policy rather than infer it from task
design or model choice. Current practice relies on indirect signals, such as
model class comparisons, environment-specific probes, or sample-complexity
bounds \citep{williams_reinforcement_2009, efroni_provable_2022,
morad_popgym_2023}. These conflate optimization, inductive bias, and true memory
demand, and they do not yield a comparable, sequence-level number.

To address this, we formalize a \emph{Temporal Range} metric that aggregates
vector-output Jacobians over lags into a magnitude-weighted average look-back,
axiomatized for uniqueness. Concretely, for each position $t$
we average Jacobian norms $\|J_{s,t}\|_{\text{mat}}$ over all subsequent final
timesteps $s\in\{t+1,\dots,T\}$, forming per-step weights
$w_t=\frac{1}{T-t}\sum_{s=t+1}^{T}\|J_{s,t}\|_{\text{mat}}$,
and report
\[
\hat{\rho}_T \;=\; \frac{\sum_{t=1}^{T} w_t\,(T-t)}{\sum_{t=1}^{T} w_t}\;\in[0, T{-}1], \qquad \text{with } \sum_{t=1}^{T} w_t>0.
\]
Thus $\hat{\rho}_T$ answers ``how far back is this policy looking \emph{here}?''
at the level of a specific rollout and timestep. Our theoretical starting point
is the axioms of range from \citet{bamberger_measuring_2025}; we specialize them
to temporal lag, extend to vector-valued outputs, and study their consequences
for reinforcement learning agents.

For vector-output \emph{linear} maps we give a short axiomatic derivation that
fixes both the unnormalized and normalized forms: single-step calibration,
additivity over disjoint time indices (or magnitude-weighted averaging), and
absolute homogeneity identify the unique matrix-norm–weighted lag sums/averages.
The same formulas applied to the local linearization yield our nonlinear policy
metric. The normalized form is invariant to uniform input rescaling (change of
units) and uniform output rescaling, making cross-agent and cross-environment
comparisons straightforward.

Computationally, Temporal Range is inexpensive. The required Jacobian blocks are
\emph{policy} derivatives with respect to observation inputs and can be obtained
with standard reverse-mode automatic differentiation on the policy alone. When
direct auto-differentiation on a given model is impractical, we train a compact
LEM policy on the same task and compute the same quantities on this proxy.

We validate Temporal Range across POPGym diagnostics and classic control with
Multi-Layer Perceptrons (MLPs), gated RNNs (e.g., LSTMs, Gated Recurrent Units
(GRUs) \citep{chung2014empirical}, and LEM), as well as SSMs. The metric (i)
stays near zero in fully observed control, (ii) scales with the ground-truth
offset in Copy-$k$, and (iii) aligns with the smallest history window needed to
achieve near-optimal return, as verified by window ablations that rebuild hidden
state from truncated histories. These results make Temporal Range a practical
tool for auditing memory use and for selecting the shortest sufficient context.

\paragraph{Contributions.}
(i) We formalize a simple, model-agnostic metric, Temporal Range, that
aggregates per-timestep \emph{matrix}-norm Jacobians of a \emph{vector} output
into a magnitude-weighted average lag, computable in one reverse-mode pass. This
captures temporal dependence and is complementary to existing saliency and
perturbation-style explanations, which focus on spatial attention within single
frames. (ii) A minimal axiomatic justification for vector-output linear maps
fixing both unnormalized and normalized forms, with invariance to uniform input
and output rescaling. (iii) Empirical validation on diagnostics and control
showing agreement with ground truth and with window-size requirements for high
return. (iv) A LEM-based proxy enabling use when the policy blocks gradients or
is black-box.

\section{Related Work}
Research on memory in RL has often focused on \emph{implementing} mechanisms
(RNNs, Transformers, memory buffers), but fewer works directly \emph{measure}
how much past information agents actually use. We review theoretical foundations
and empirical approaches to quantifying memory dependence.

\noindent\textbf{Foundational context.} Learning long-range dependencies is
nontrivial due to exploding/vanishing gradients and bias from truncated BPTT
\citep{pascanu2013difficulty,tallec2018recurrentneuralnetworkswarp}. Our focus
on history use is motivated by the classical view that partial observability
demands memory \citep{kaelbling_planning_1998} and by temporal credit-assignment
foundations in RL \citep{sutton_reinforcement_2018}. Computationally, our
reverse-mode computation of temporal influence ties directly to backpropagation
through time \citep{werbos_backpropagation_1990}. Empirically, our scalar
Temporal Range parallels “effective context length” observations in language
modeling \citep{khandelwal_sharp_2018}. Finally, since our metric aggregates
Jacobian magnitudes, we note standard cautions from the saliency literature
\citep{adebayo_sanity_2018}.

\subsection{Theoretical Frameworks}
Classical Markov decision processes (MDPs) assume the Markov property: future
states depend only on the present. Partially observed Markov decision processes
(POMDPs), however, are non-Markovian to the agent and thus require memory.
\citet{mizutani_totally_2017} show that recurrent architectures can render
non-Markovian processes Markovian in an augmented state space.
\citet{efroni_provable_2022} prove sample-complexity bounds for POMDPs where
latent states can be decoded from histories of length $m$, showing exponential
scaling in $m$. Other work links partial observability and memory through
Bayesian/active inference perspectives \citep{malekzadeh_active_2024}, and
extends to decentralized multi-agent settings
\citep{omidshafiei_decentralized_2017}.

\subsection{Empirical Measurement Approaches}
\textbf{Model comparisons.} POPGym \citep{morad_popgym_2023} provides partially
observable tasks for comparing memory-augmented models.

\textbf{Environment-specific studies.} Modified CartPole
\citep{koffi_novel_2020}, Pacman \citep{fallasmoya_measuring_2021}, and
multi-robot delivery \citep{omidshafiei_decentralized_2017} isolate memory
requirements in specific tasks.

\textbf{Performance gaps.} Studies such as \citet{meng_memory_2021} compare
memoryless Twin Delayed Deep Deterministic Policy Gradient (TD3) vs. recurrent
(LSTM-TD3) agents in MDPs vs. POMDPs, quantifying degradation without memory.
\citet{williams_reinforcement_2009} give conditions where stochastic memoryless
policies suffice.

\textbf{Robustness and complexity.} Metrics such as adversarial perturbation
robustness \citep{zhang_robust_2020} and theoretical sample-complexity bounds
\citep{efroni_provable_2022} indirectly quantify the cost of memory.

\paragraph{Gradient-based attribution and deep RL.} Our metric builds on the
classical saliency view that interprets Jacobian magnitudes, $|\partial
y/\partial x|$, as input importance \citep{simonyan_deep_2014}. In deep RL,
prior work uses such saliency/perturbation maps to show \emph{where} a vision
policy attends within a single frame \citep{greydanus_visualizing_2018}, but
these analyses are inherently \emph{spatial} and per-timestep and often note
that raw Jacobian maps can be visually uninformative. Gradient-based
attributions admit known caveats; complementary baselines include Integrated
Gradients, SmoothGrad, and influence-function–style analyses
\citep{sundararajan2017axiomaticattributiondeepnetworks,smilkov2017smoothgradremovingnoiseadding,koh2020understandingblackboxpredictionsinfluence,pruthi2020estimatingtrainingdatainfluence}.

\paragraph{Temporal Range versus saliency.}
By contrast, Temporal Range is explicitly \emph{temporal}: we aggregate
matrix-norm Jacobian blocks over past observations to produce a scalar that
quantifies \emph{how far back} a policy depends on history. Beyond attribution,
we supply (i) an axiomatic justification that fixes the form of the metric and
(ii) behavioral validation via window ablations aligning measured range with the
minimal history needed for high return. To our knowledge, no prior
saliency-based method offers a scalar, sequence-level readout of memory
dependence with these guarantees \citep{simonyan_deep_2014,
greydanus_visualizing_2018}.

\subsection{Benchmark Development}
Benchmarks provide standardized evaluation. POPGym \citep{morad_popgym_2023} is
widely used, while newer efforts like MIKASA \citep{cherepanov_memory_2025}
classify memory-intensive RL tasks such as robotic manipulation.

\section{A Temporal-Range Measure for Reinforcement Learning}
\label{sec:temporal_range}
How much does a policy at time $T$ look back in the past? Our goal is a
concrete, \emph{model-agnostic} number. The plan: treat first-order
sensitivities of the final \emph{vector} output with respect to earlier inputs
as a \emph{temporal influence profile}, and summarize that profile by the
expected look-back (average lag). We then show this summary is not arbitrary: in
the linear case, it is the \emph{unique} object that satisfies a small set of
first-principles requirements for any reasonable “how-far-back” score. For
nonlinear policies, we apply the same formula to the local linearization.

\subsection{Setup and definition}
Let $F:\mathbb{R}^{T\times d}\to\mathbb{R}^c$ denote the map
 from a length-$T$ observation sequence to a vector output, $X_{1:T}\mapsto
 y(X)$, where $X_{1:T}=[x_1,\dots,x_T]$ with $x_t\in\mathbb{R}^d$ and
 $y(X)\in\mathbb{R}^c$ is the policy's vector output (e.g., action logits or
 probabilities). For each pair $s,t$ with $1\le t < s\le T$,
 write the Jacobian block
\[
J_{s,t}(X)\;\coloneqq\;\frac{\partial y_s(X)}{\partial x_t}\in\mathbb{R}^{c\times d},
\]
where $y_s(X)$ is the output at step $s$ given the partial
sequence $X_{1:s}$. Fix a matrix norm $\|\!\cdot\!\|_{\text{mat}}$ on
$\mathbb{R}^{c\times d}$ (e.g., Frobenius, or an induced operator norm).
For a window of length $T$, define the per-step
\emph{influence weight} at position $t$ by aggregating the
Jacobian norms over all subsequent timesteps $s\in\{t+1,\dots,T\}$ via an operator $\bigoplus$:
\begin{align}
w_t(X) &\coloneqq \bigoplus_{s=t+1}^{T} \|J_{s,t}(X)\|_{\text{mat}}\;\;\;\ge 0.
\end{align}
The operator $\bigoplus$ is configurable: we use $\bigoplus = \text{mean}$ (i.e., $\frac{1}{T-t}\sum_{s=t+1}^{T}$) throughout this paper, as it provides better discrimination between tasks with different memory requirements. For tasks with concentrated temporal dependencies (e.g., a single critical observation affecting one future action), $\bigoplus = \max$ may be more appropriate; see Appendix~\ref{app:max-aggregation} for a comparison.
In practice, we compute $w_t$ by averaging over a calibration set of multiple rollouts to reduce variance.
The \emph{lag} is defined by
\begin{align}
\ell(t) &\coloneqq T-t\in\{1,\dots,T{-}1\}.
\end{align}

The map $t\mapsto w_t(X)$ can be interpreted as a nonnegative temporal influence
profile; we summarize it by the expected lag under these weights.

Note that throughout, we assume the vector outputs across the
window are not independent of the input sequence $X_{1:T}$ at the evaluation
point $X$, i.e., at least one Jacobian block $J_{s,t}(X)$ is nonzero for some
$s,t$. Equivalently, $\sum_{t=1}^{T} w_t(X) > 0$.

\begin{definition}[Temporal range]
\label{def:range}
For differentiable $F$ and an evaluation sequence $X$,
\begin{align}
\rho_T(F;X) &\;\coloneqq\;\sum_{t=1}^{T} w_t(X)\,\ell(t), \label{eq:unnormalized_range}\\[2pt]
\hat{\rho}_T(F;X) &\;\coloneqq\;
\dfrac{\sum_{t=1}^{T} w_t(X)\,\ell(t)}{\sum_{t=1}^{T} w_t(X)}. \label{eq:normalized_range}
\end{align}
\end{definition}

It follows directly that $0\le \hat{\rho}_T(F;X)\le T{-}1$. The unnormalized
$\rho_T$ aggregates both \emph{how far back} and \emph{how strongly} the past
matters; the normalized $\hat{\rho}_T$ reports an average look-back in steps.

\subsection{Axiomatic basis}
\label{sec:temporal_range_axioms}

We introduce \emph{Temporal Range} and give a short axiomatic justification
tailored to time and vector outputs. Our axioms follow the style of prior
``range'' functionals but are adapted to temporal lag $\ell(t)=T{-}t$ and matrix
norms on vector outputs, yielding in the \emph{linear} case
$L(z_1,\dots,z_T)=\sum_{t=1}^{T} B_t z_t$ with $B_t\!\in\!\mathbb{R}^{c\times
d}$ the unique forms:
\begin{equation}
\rho_T(L)=\sum_{t=1}^{T}\|B_t\|_{\text{mat}}\,(T{-}t),
\qquad
\hat\rho_T(L)=
\frac{\sum_{t=1}^{T}\|B_t\|_{\text{mat}}\,(T{-}t)}
     {\sum_{t=1}^{T}\|B_t\|_{\text{mat}}}.
\label{eq:rho-linear-forms}
\end{equation}
We then use these linear forms at the local linearization of a nonlinear policy.
Full axioms (our temporal/vector-output variant) and uniqueness proofs appear in
App.~\ref{app:axioms-uniqueness}. See \citet{bamberger_measuring_2025} for
related axiomatization style.

\subsection{From linear maps to policies}
\label{sec:linear-to-policy}

Given a differentiable policy that produces vector outputs
$y_s(X)\in\mathbb{R}^c$ at each step $s$ in a window of length $T$, define the
Jacobian blocks and their averaged matrix–norm magnitudes
\begin{equation}
\label{eq:policy-jac}
J_{s,t}(X)\;=\;\frac{\partial y_s}{\partial x_t}(X)\in\mathbb{R}^{c\times d},
\qquad
w_t(X)\;=\;\frac{1}{T-t}\sum_{s=t+1}^{T}\bigl\|J_{s,t}(X)\bigr\|_{\text{mat}}.
\end{equation}
Plugging these weights into \eqref{eq:rho-linear-forms} yields
\begin{equation}
\label{eq:policy-range}
\begin{aligned}
\rho_T(F;X) \;&=\; \sum_{t=1}^{T} w_t(X)\,(T-t),\\
\hat{\rho}_T(F;X) \;&=\; \frac{\sum_{t=1}^{T} w_t(X)\,(T-t)}{\sum_{t=1}^{T} w_t(X)}\,,
\end{aligned}
\end{equation}
the magnitude-weighted average look-back of the local linearization at $X$.

\subsection{Invariances and scope}
\paragraph{Uniform output rescaling.}
If $\tilde y_s=\alpha\,y_s$ for all $s$ with $\alpha\neq 0$,
then $J_{s,t}(\tilde y)=\alpha\,J_{s,t}(y)$ for all $s,t$, so every term in the
averaging sum $w_t=\frac{1}{T-t}\sum_{s=t+1}^{T}\|J_{s,t}\|_{\text{mat}}$ is
multiplied by $|\alpha|$, leaving $\hat{\rho}_T$ unchanged (while $\rho_T$
rescales by $|\alpha|$). This covers, e.g., temperature scaling of logits.

\paragraph{Uniform input rescaling (change of units).}
If we change units by $x_t^\star=\beta x_t$ with $\beta\neq 0$,
then by the chain rule $\|\frac{\partial y_s}{\partial
x_t^\star}(X)\|_{\text{mat}} =\frac{1}{|\beta|}\,\|\frac{\partial y_s}{\partial
x_t}(X)\|_{\text{mat}}$ for all $s,t$, so every term in $w_t$ is multiplied by
$1/|\beta|$, leaving $\hat{\rho}_T$ invariant (while $\rho_T$ rescales by
$1/|\beta|$). This concerns a \emph{reparameterization of inputs}; feeding
$\beta X$ into a fixed network changes $F$ and may change $\hat{\rho}_T$.



\subsection{Immediate RL consequences}
\label{sec:rl-consequences}

Temporal Range serves primarily as an interpretability tool: it
quantifies how far back a trained policy looks when making decisions. In
partially observed or noisy settings, earlier inputs matter for state
reconstruction and smoothing, so $\hat{\rho}_T$ grows. Comparing architectures,
short ranges on tasks believed to need memory may indicate under-capacity; long
ranges on near-Markov tasks indicate unnecessary slow modes (common with some
SSMs). The measure also provides an upper bound for sizing history windows:
aggregating $\hat{\rho}_T$ across evaluation episodes suggests a conservative
context length, which window ablations can then validate.

Temporal Range is a local, per-sequence diagnostic: at a given rollout
window it asks how far back the policy is effectively looking.
In fully observed control, if each $y_s$ depends only on
$x_s$, then $J_{s,t}(X)=0$ for all $t<s$, so $w_t=0$ and
$\hat{\rho}_T=0$ at that point. With a finite effective memory $(x_{s-m+1{:}s})$,
$J_{s,t}=0$ for $t\le s-m$ across all $s$ and
$\hat{\rho}_T\in[0,m{-}1]$; for a single-offset dependence (Copy-$k$),
$\hat{\rho}_T=k$ exactly. These cases calibrate scale and serve as unit tests.

Policy and value heads can differ materially. Applying the same computation to
the value head $V_T$ often yields a larger effective range, reflecting return
propagation; a gap between policy and value ranges helps explain unstable
advantage estimates and motivates retuning temporal credit assignment (e.g., the
$\lambda$ in Generalized Advantage Estimation (GAE)
\citep{schulman2016highdimensionalcontinuouscontrolusing}). Importantly, uniform
output rescaling (e.g., logits temperature) multiplies all $w_t$ by the same
factor and leaves the normalized range invariant, so exploration via temperature
does not confound comparisons of $\hat{\rho}_T$.

\subsection{Analytical calibration in toy differentiable settings}
\label{sec:analytical_toy_blurb}
When the end-to-end map happens to be fully differentiable and structured,
Temporal Range admits closed-form expressions that calibrate the measure. In
App.~\ref{sec:analyticalTR} we work out two cases: (\textit{i})
\textsc{Copy-$k$}, where $\hat{\rho}_T=k$ exactly for any matrix norm, and
(\textit{ii}) linear recurrent readouts, where $w_t=\|Q A^{T-t}
C\|_{\text{mat}}$ yields a profile governed by the spectrum of $A$. These
derivations are useful for calibration and intuition. In practical RL
environments, we compute Jacobians via reverse-mode \emph{on the policy}. If the
target model is not amenable to automatic differentiation, we compute the same
quantities on a compact LEM proxy (Sec.~\ref{sec:lem}).

\section{Approximating Temporal Range with LEM}
\label{sec:lem}

Computing Temporal Range requires differentiating the policy outputs
$y_s$ for $s\in\{1,\dots,T\}$ with respect to past
\emph{observations} $x_{1:T}$. This does not require differentiating the
simulator. In settings where the policy itself is non-differentiable or
inaccessible to autograd, we train a compact LEM policy on the same task and
compute TR on this proxy. We use Long Expressive Memory (LEM)
\citep{rusch_lem_2022} as an effective proxy thanks to its stable long-horizon
gradients. LEM is designed for long-horizon sequence modeling and maintains
stable gradients over extended contexts via a multiscale ordinary differential
equation (ODE) formulation.

\paragraph{Empirical check.}  
In benchmarks, proxy $\hat{\rho}_T$ varies in line with known task structure: it
increases with $k$ in Copy-$k$ and grows under partial observability, indicating
that the proxy yields reliable Jacobians when the original model is not amenable
to differentiation.

\section{Experiments}
\label{sec:experiments}

We quantify how much trained agents use history by measuring \emph{Temporal
Range} $\hat{\rho}_T$ and validating it with \emph{window ablations}. We
evaluate four architectures, namely LEM, GRU, LSTM, and Linear Oscillatory
State-Space models (LinOSS) \citep{rusch2025linoss}, across diagnostic and
control settings
(Tables~\ref{tab:temporal_range_meanstd}--\ref{tab:performance_norm}).

\subsection{Setups}
\textbf{Training.}
All policies are trained with Proximal Policy Optimization (PPO)
(\citep{schulman_proximal_2017}) for $10^7$ steps. Unless stated, the
actor--critic uses a single LEM cell (size 128) with 64-d encoder/decoder MLPs;
GRU/LSTM and LinOSS are dimension-matched. We use the JAX auto-differentiation
framework \citep{jax2018github}.

\textbf{Temporal-Range computation.}
At evaluation, for each rollout of length $T$ we compute Jacobian blocks
$J_{s,t}=\partial y_s/\partial x_t\in\mathbb{R}^{c\times d}$
for all pairs $s,t$ with $t<s\le T$, where $y_s$ is the \emph{action-vector}
output at step $s$. We then average these norms over final timesteps to form
per-step weights $w_t=\frac{1}{T-t}\sum_{s=t+1}^{T}\|J_{s,t}\|_{\text{mat}}$
and aggregate into $\hat{\rho}_T$ (Def.~\ref{def:range}). Unless otherwise
noted, we use the Frobenius norm. We further average over
multiple rollouts (calibration set) and report means and standard deviations
over episodes.

\textbf{Window ablations (behavioral validation).}
To test whether measured look-back is functionally required, we evaluate each
trained policy under truncated histories of size $m\in\{1,2,4,8,16,32,64\}$: at
every step, we rebuild hidden state from the last $m$ observations only. Curves
(return vs.\ $m$) and the summary (Best@m, Avg.) are given in
Fig.~\ref{fig:ablations} and Table~\ref{tab:window_ablation_summary}.

\subsection{Environments}
\label{sec:envs}

We evaluate agents on a small suite of POPGym (\citep{morad_popgym_2023})
environments designed to probe distinct memory behaviors:

\begin{itemize}[leftmargin=1.25em,itemsep=2pt,topsep=2pt]

\item \textbf{Control (fully vs.\ partially observed).} CartPole is nearly
Markov and should require little history. In \emph{Stateless CartPole},
positions are hidden, so the agent must integrate observations. The \emph{Noisy
Stateless} variant adds observation noise, increasing smoothing demand.

\item \textbf{Diagnostics (known lags).} \emph{Repeat First} asks the agent to
recall an early value later in the episode. \emph{Copy-$k$} (with
$k\in\{1,3,5,10\}$) requires outputting the observation from exactly $k$ steps
ago, providing ground-truth offsets.

\end{itemize}

\section{Results}
\label{sec:results}

We now present empirical results. Tables report summary statistics across
environments and architectures, while plots pair reward traces with temporal
influence profiles to visualize both performance and measured look-back. Unless
otherwise noted, results use LEM as the representative model; the appendix
contains full sweeps and additional figures.

\subsection{Summary across settings}
Figure~\ref{fig:ablations} (ablations) and Figure~\ref{fig:profiles} (profiles)
show the same tasks side by side. On \emph{CartPole} (near-Markov), effective
look-back is short (GRU ${\sim}1$, LEM ${\sim}3$;
Table~\ref{tab:temporal_range_meanstd}), and returns saturate at small windows,
matching the compact profiles. In \emph{Stateless CartPole}, profiles develop
longer tails and ablations recover only with larger windows. In \emph{Copy-$k$},
profiles shift right with $k$ and ablations exhibit knees near
$\hat{\rho}_T{+}1$, while calibration MAE (Table~\ref{tab:copyk_mae}) is lowest
for GRU at higher $k$, with LEM/LinOSS tending to overestimate due to
slow/multiscale modes.

\begin{figure}[t]
  \centering

  \begin{subfigure}[t]{0.48\textwidth}
    \centering
    \includegraphics[width=\linewidth]{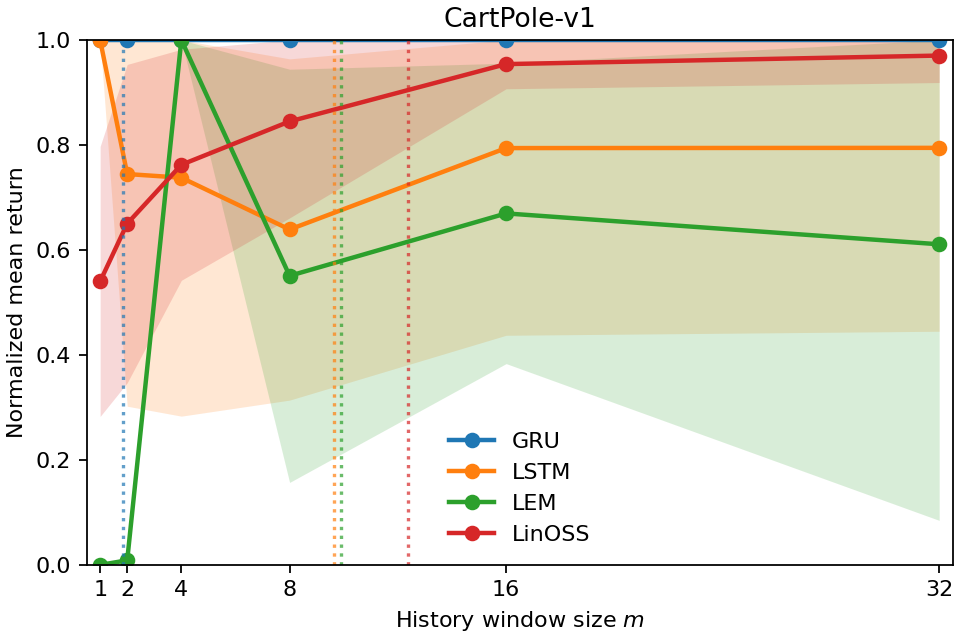}
    \caption{CartPole (near-Markov)}
    \label{fig:abl-cartpole}
  \end{subfigure}\hfill
  \begin{subfigure}[t]{0.48\textwidth}
    \centering
    \includegraphics[width=\linewidth]{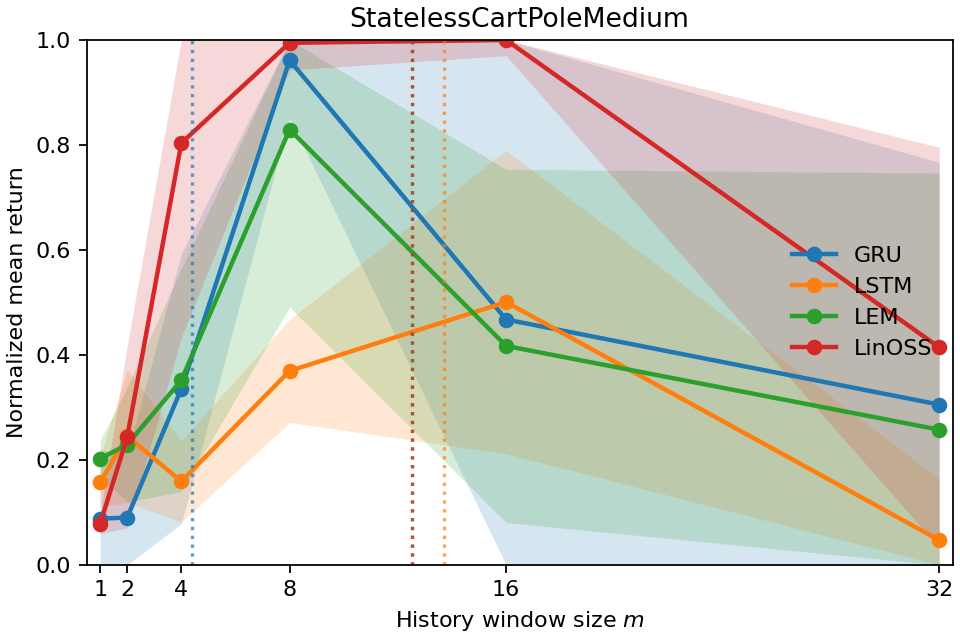}
    \caption{Stateless CartPole}
    \label{fig:abl-stateless}
  \end{subfigure}

  \medskip

  \begin{subfigure}[t]{0.48\textwidth}
    \centering
    \includegraphics[width=\linewidth]{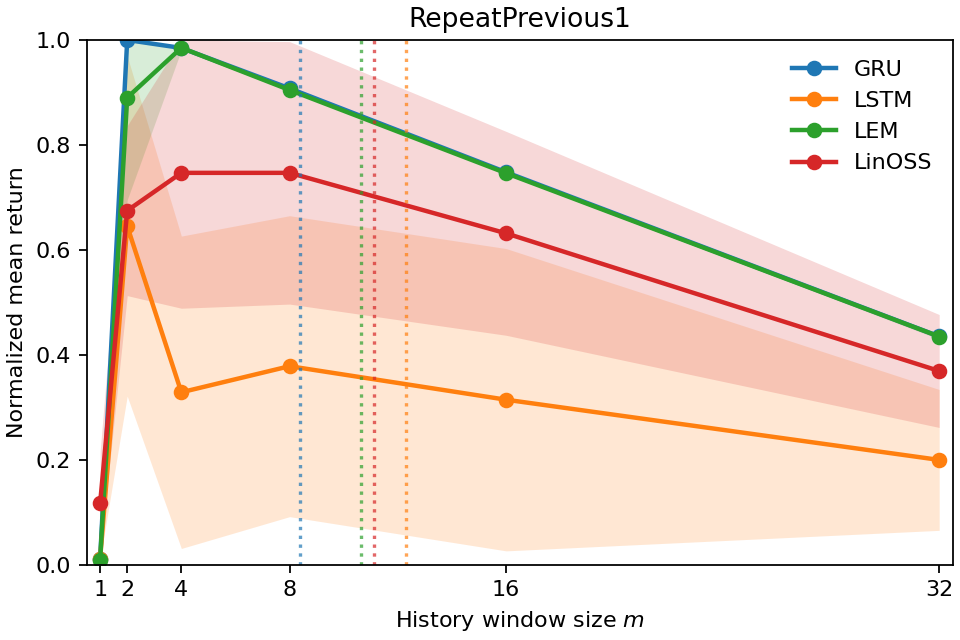}
    \caption{Copy $k=1$}
    \label{fig:abl-copy1}
  \end{subfigure}\hfill
  \begin{subfigure}[t]{0.48\textwidth}
    \centering
    \includegraphics[width=\linewidth]{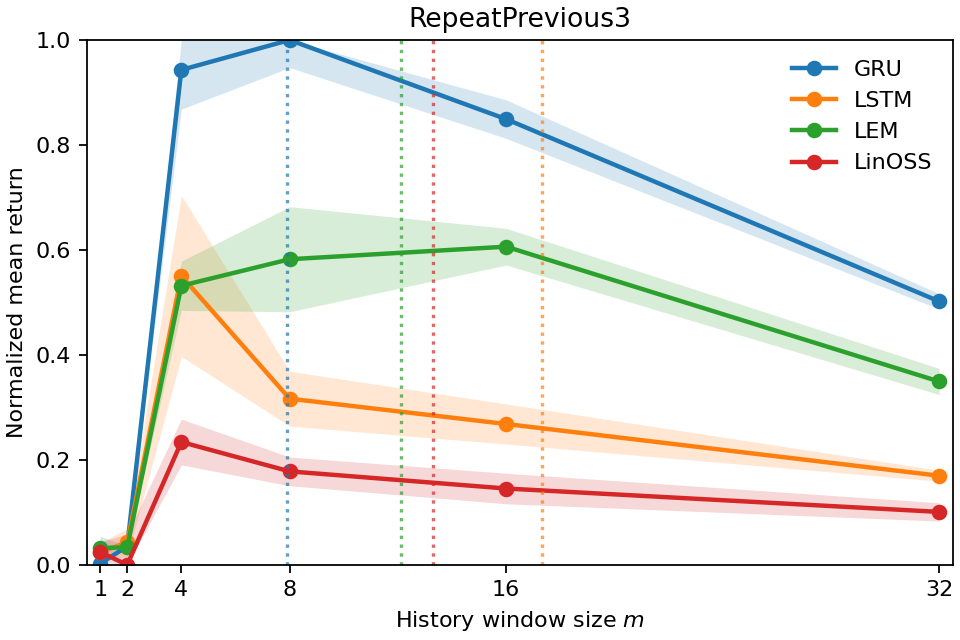}
    \caption{Copy $k=3$}
    \label{fig:abl-copy3}
  \end{subfigure}

  \medskip

  \begin{subfigure}[t]{0.48\textwidth}
    \centering
    \includegraphics[width=\linewidth]{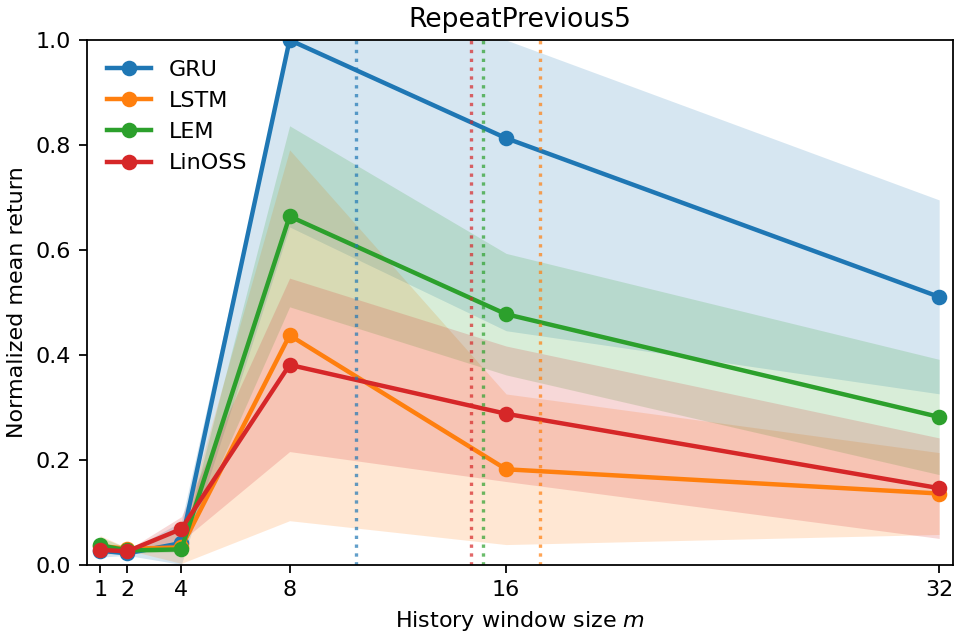}
    \caption{Copy $k=5$}
    \label{fig:abl-copy5}
  \end{subfigure}\hfill
  \begin{subfigure}[t]{0.48\textwidth}
    \centering
    \includegraphics[width=\linewidth]{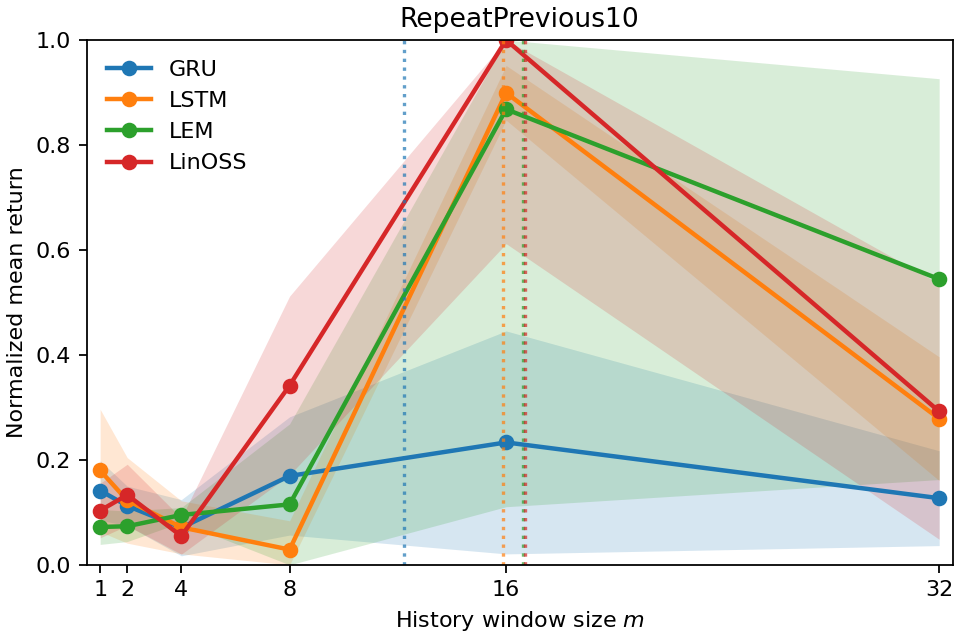}
    \caption{Copy $k=10$}
    \label{fig:abl-copy10}
  \end{subfigure}

  \caption{\textbf{Window ablations.} Normalized return vs.\
  context window size $m$ across architectures. Dotted vertical lines show
  $\hat{\rho}_T$ values. Performance recovers when $m$ exceeds temporal range,
  confirming TR identifies minimum sufficient context. Note that $\hat{\rho}_T$
  often aligns remarkably well with task requirements (e.g., GRU's $\hat{\rho}_T
  \approx 12$ for Copy $k=10$). When TR appears to fall short of the empirical
  peak (e.g., $\hat{\rho}_T=12$ while peak occurs at $m=16$), this is typically
  an artifact of our sparse window sampling ($m \in \{1,2,4,8,16,32\}$); the
  true performance peak likely lies between tested values, closer to the TR
  prediction.}
  \label{fig:ablations}
\end{figure}

\begin{figure}[t]
  \centering

  \begin{subfigure}[t]{0.48\textwidth}
    \centering
    \includegraphics[width=\linewidth]{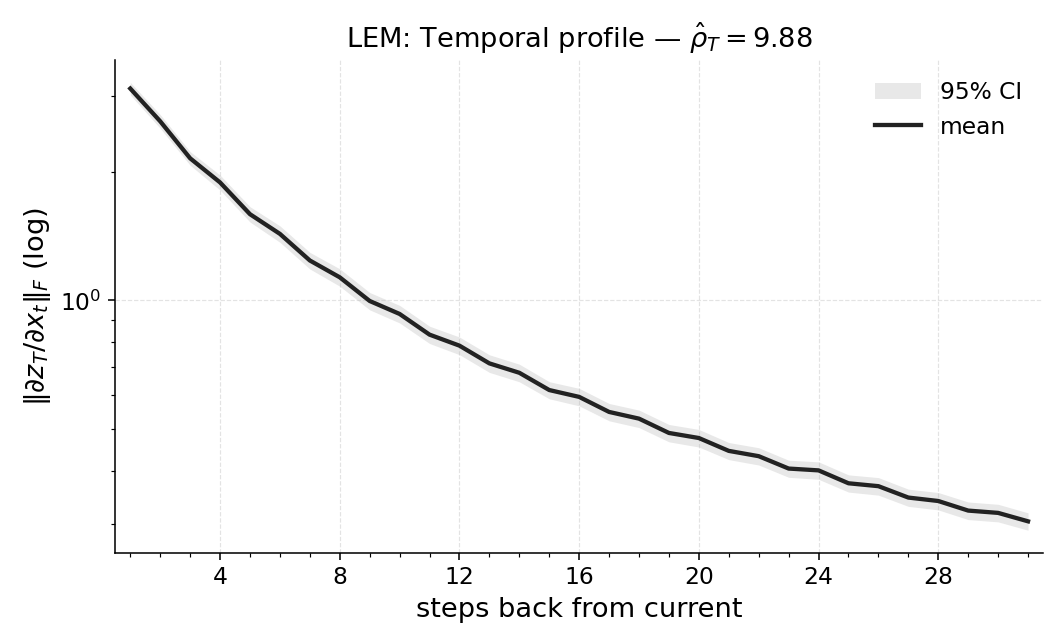}
    \caption{CartPole (profile)}
    \label{fig:prof-cartpole}
  \end{subfigure}\hfill
  \begin{subfigure}[t]{0.48\textwidth}
    \centering
    \includegraphics[width=\linewidth]{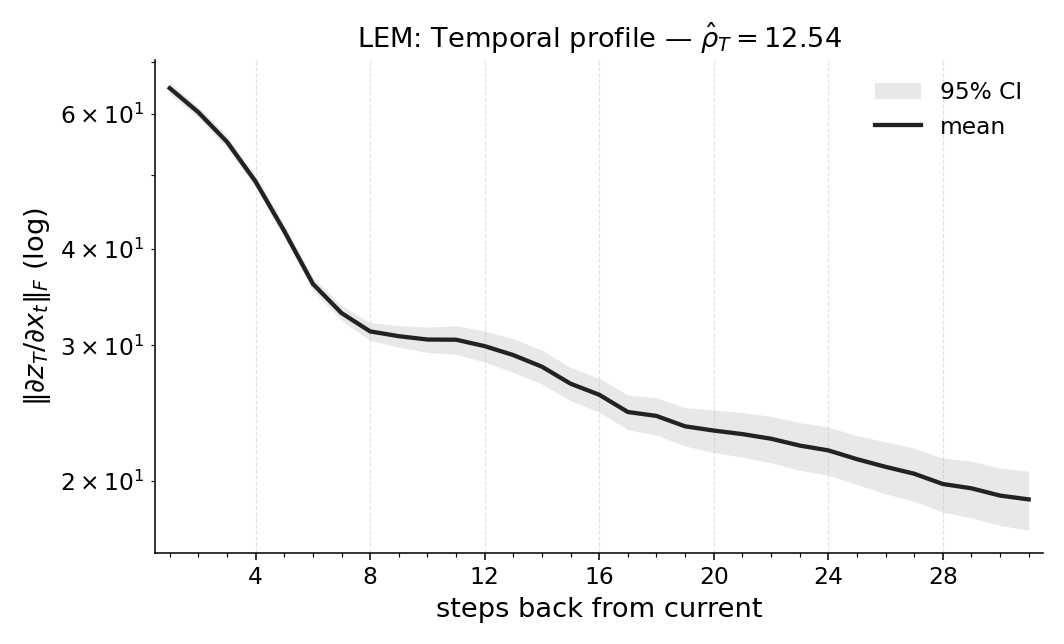}
    \caption{Stateless CartPole (profile)}
    \label{fig:prof-stateless}
  \end{subfigure}

  \medskip

  \begin{subfigure}[t]{0.48\textwidth}
    \centering
    \includegraphics[width=\linewidth]{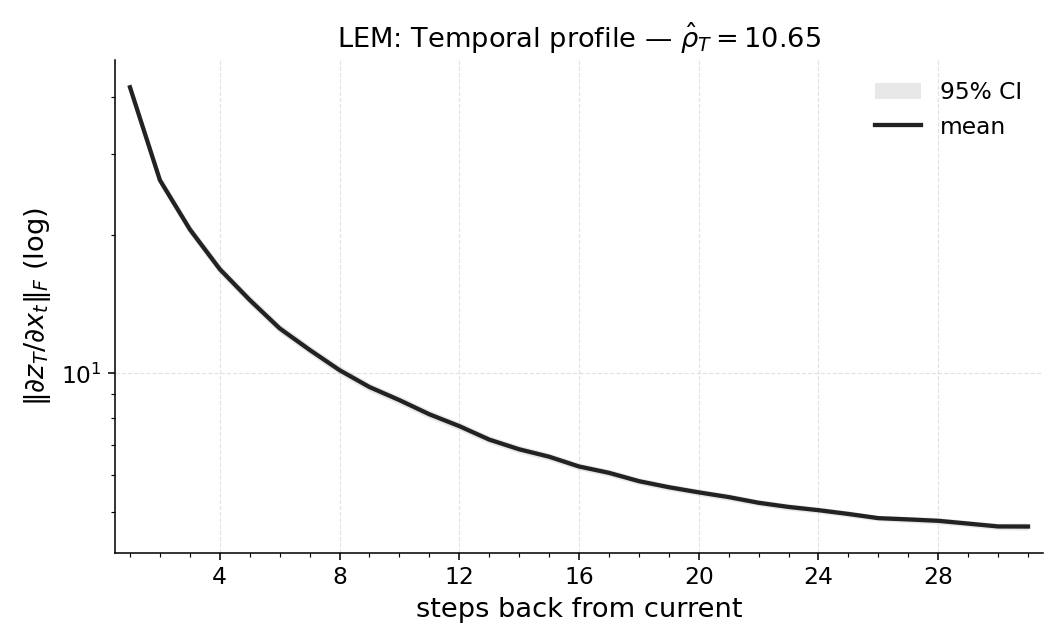}
    \caption{Copy $k=1$ (profile)}
    \label{fig:prof-copy1}
  \end{subfigure}\hfill
  \begin{subfigure}[t]{0.48\textwidth}
    \centering
    \includegraphics[width=\linewidth]{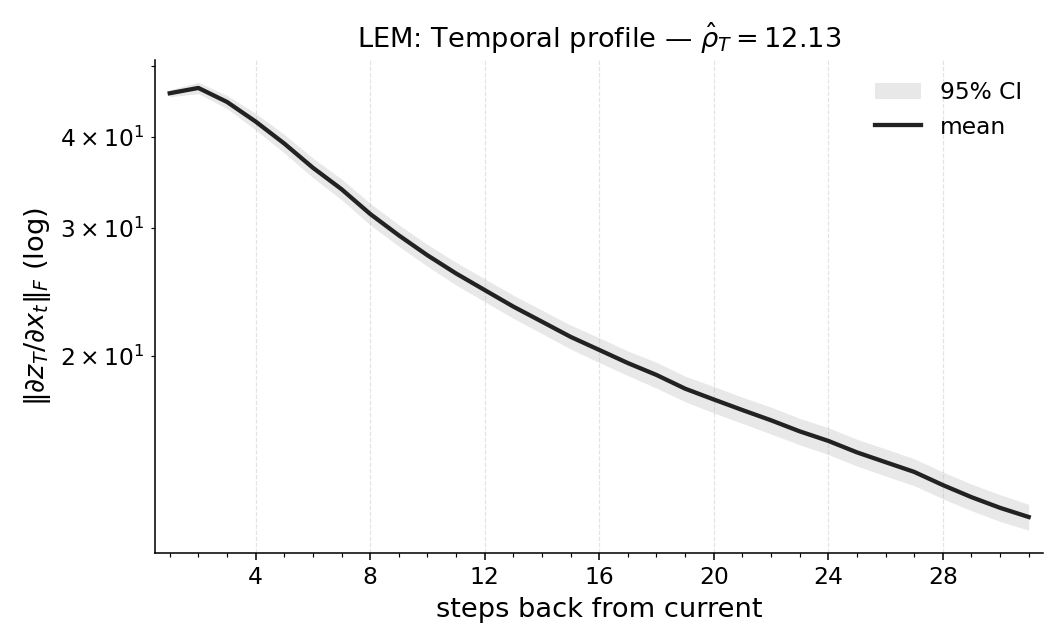}
    \caption{Copy $k=3$ (profile)}
    \label{fig:prof-copy3}
  \end{subfigure}

  \medskip

  \begin{subfigure}[t]{0.48\textwidth}
    \centering
    \includegraphics[width=\linewidth]{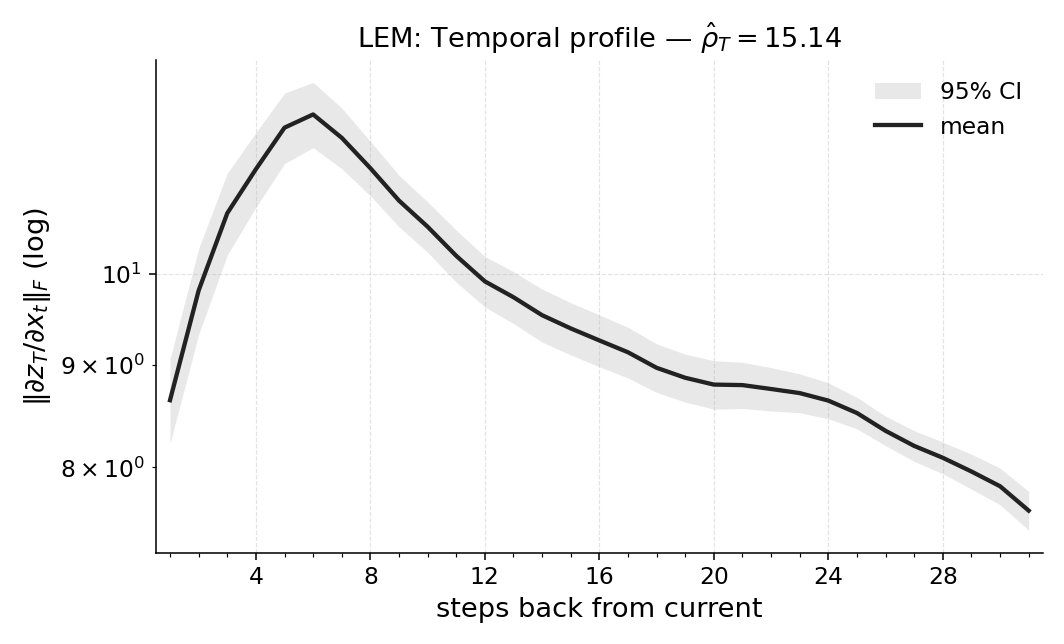}
    \caption{Copy $k=5$ (profile)}
    \label{fig:prof-copy5}
  \end{subfigure}\hfill
  \begin{subfigure}[t]{0.48\textwidth}
    \centering
    \includegraphics[width=\linewidth]{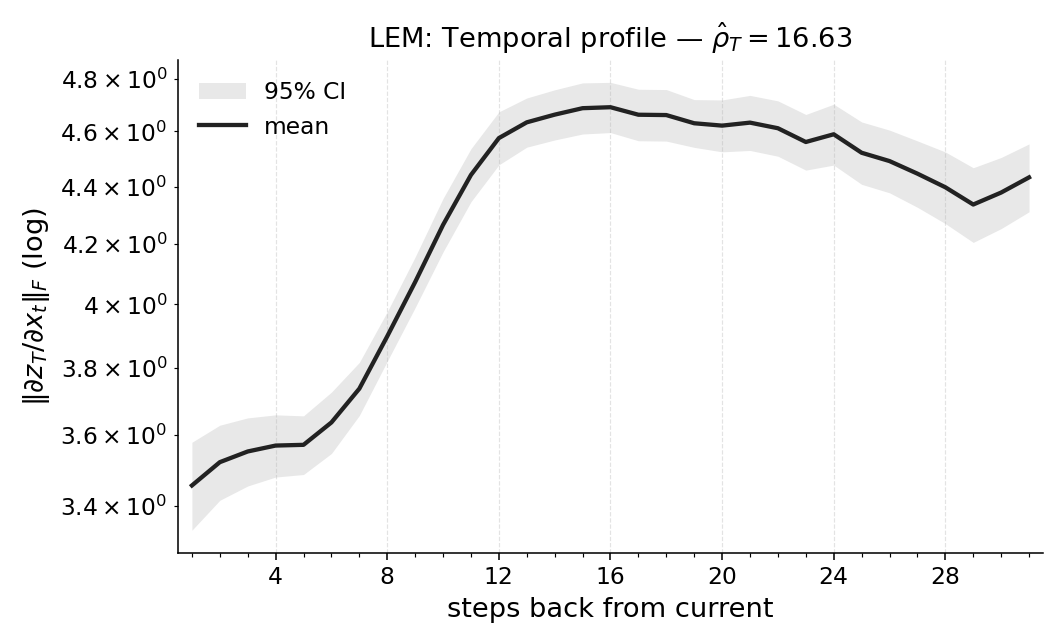}
    \caption{Copy $k=10$ (profile)}
    \label{fig:prof-copy10}
  \end{subfigure}

  \caption{\textbf{Temporal influence profiles (LEM).} Jacobian
  magnitude vs.\ steps back from current timestep. Profiles show how policy
  depends on observation history: CartPole concentrates on recent steps,
  Stateless CartPole distributes broadly, Copy-$k$ peaks near required lookback
  distance.}
  \label{fig:profiles}
\end{figure}

\subsection{Key observations}
Range tracks observability: it is small in \emph{CartPole} and increases in
\emph{Stateless CartPole}, with ablations recovering only once windows exceed
the measured look-back (Figs.~\ref{fig:ablations}, \ref{fig:profiles};
Tables~\ref{tab:temporal_range_meanstd}, \ref{tab:window_ablation_summary}). In
\emph{Copy-$k$}, $\hat{\rho}_T$ scales with $k$ and the ablation knee aligns
with ground truth (Tables~\ref{tab:temporal_range_meanstd},
\ref{tab:copyk_mae}). Notably, even in near-Markovian CartPole,
restricting context to only the current observation ($m=1$) causes performance
degradation, likely due to distribution shift; models trained with full history
adapt to using second-to-last observations, making sudden truncation
disruptive. Interestingly, TR reveals architecture-specific inefficiencies: GRU achieves $\hat{\rho}_T{\approx}2$ on CartPole, while LEM and LSTM hover around $10$, indicating they genuinely rely on longer history despite the task's near-Markovian nature. This highlights TR's value as an interpretability tool for exposing how different architectures use memory, even when such usage is unnecessary for performance. Across architectures, GRU tends to shorter effective memory unless
forced; LEM maintains multiscale tails with stable behavior; LinOSS often
inflates range via slow modes on easy tasks; LSTM is more seed-sensitive.

\subsection{TR-guided memory-efficient deployment}
\label{sec:tr-deployment}

To validate TR's practical utility for architecture design, we
test whether TR-recommended context windows enable memory-efficient deployment
while preserving performance. We follow a two-stage protocol: (1) train small
LEM models (hidden size 128) to compute TR estimates $\hat{\rho}_T$, (2) train
large LEM models (hidden size 512) with full continuous hidden state evolution.
We then evaluate the large models under three regimes: \emph{training}
(continuous evolution baseline), \emph{TR window} (hidden states rebuilt from
observation buffers of length $\lceil\hat{\rho}_T{+}1\rceil$), and
\emph{half-TR} (buffers of length $\lceil(\hat{\rho}_T{+}1)/2\rceil$). This
tests whether windowed evaluation approximates the model's natural operation
while reducing memory requirements.

Performance is normalized to $[0,1]$ per environment using
global min/max values from all training runs, with results averaged over
multiple evaluation trials.

\begin{table}[h]
\centering
\small
\caption{TR-guided window validation: normalized performance [0,1]
(mean $\pm$ std over evaluation trials). Training baseline shows final
performance with continuous hidden state evolution. Retention indicates windowed
performance relative to training baseline.}
\label{tab:tr_deployment}
\begin{tabular}{lccccc}
\toprule
Environment & TR & Training & TR window & Half-TR & Retention \\
\midrule
Noisy Stateless CartPole & $12.3$ & $0.474\pm0.017$ & $0.953\pm0.068$ &
$0.637\pm0.116$ & $201.1\%$ \\
Copy $k{=}3$ & $12.1$ & $0.934\pm0.007$ & $0.970\pm0.038$ & $0.838\pm0.021$ &
$103.9\%$ \\
Copy $k{=}10$ & $16.6$ & $0.561\pm0.018$ & $0.979\pm0.034$ & $0.052\pm0.027$ &
$174.1\%$ \\
\bottomrule
\end{tabular}
\end{table}

The results demonstrate that TR-recommended windows
successfully approximate training-regime performance despite the distribution
shift from periodic hidden state reconstruction. Models evaluated with TR-guided
truncated windows maintain high retention relative to continuous evolution
(103.9--201.1\%), validating that TR identifies sufficient context for the
windowed regime. In contrast, half-TR windows show substantial degradation,
confirming that truncating below TR recommendations loses critical temporal
information. These findings support the hypothesis that TR computed from small
models (h=128) provides actionable upper bounds for deploying large models
(h=512) with reduced memory. The ability to maintain performance while
reconstructing hidden states from windows of 10--20 steps (versus continuous
evolution over full episodes) demonstrates TR's utility for memory-constrained
deployment scenarios.

\begin{table}[t]
\centering
\small
\setlength{\tabcolsep}{7pt}
\caption{Temporal Range $\hat{\rho}_T$ (mean $\pm$ std steps).
Values increase with task memory requirements where policies are competent: moderate
for near-Markovian CartPole (with architecture-specific variation), higher for stateless variants, and scaling with $k$
in Copy tasks.}
\label{tab:temporal_range_meanstd}
\begin{tabular}{lcccc}
\toprule
Environment  & LEM & GRU & LSTM & LinOSS \\
\midrule
CartPole & $9.883\pm0.358$ & $1.854\pm1.842$ & $9.642\pm0.818$ &
$12.362\pm0.377$ \\
Stateless CartPole & $12.536\pm1.428$ & $4.391\pm1.000$ & $13.704\pm2.048$ &
$12.514\pm1.101$ \\
Noisy Stateless CartPole & $12.261\pm0.637$ & $9.762\pm1.388$ & $15.274\pm1.342$
& $14.151\pm0.488$ \\
RepeatFirst & $17.793\pm2.536$ & $3.742\pm2.504$ & $12.489\pm2.174$ &
$21.177\pm0.813$ \\
Copy $k=1$ & $10.647\pm0.424$ & $8.398\pm1.552$ & $12.294\pm0.802$ &
$11.111\pm0.296$ \\
Copy $k=3$ & $12.126\pm0.513$ & $7.916\pm0.498$ & $17.312\pm0.880$ &
$13.298\pm0.956$ \\
Copy $k=5$ & $15.137\pm0.588$ & $10.453\pm1.538$ & $17.255\pm1.176$ &
$14.693\pm0.670$ \\
Copy $k=10$ & $16.625\pm0.585$ & $12.224\pm1.376$ & $15.900\pm3.403$ &
$16.715\pm0.652$ \\
\bottomrule
\end{tabular}
\end{table}

\begin{table}[t]
\centering
\small
\setlength{\tabcolsep}{7pt}
\caption{Policy performance (normalized return in [0,1], mean
$\pm$ std). After hyperparameter tuning, all architectures achieve strong
performance, enabling meaningful TR analysis across models.}
\label{tab:performance_norm}
\begin{tabular}{lcccc}
\toprule
Environment  & LEM & GRU & LSTM & LinOSS \\
\midrule
CartPole & $0.989\pm0.002$ & $0.998\pm0.003$ & $0.960\pm0.029$ & $0.914\pm0.135$
\\
Stateless CartPole  & $0.892\pm0.186$ & $0.999\pm0.002$ & $0.926\pm0.065$ &
$0.796\pm0.320$ \\
Noisy Stateless CartPole& $0.534\pm0.010$ & $0.492\pm0.027$ & $0.408\pm0.022$ &
$0.469\pm0.063$ \\
RepeatFirst & $0.541\pm0.033$ & $0.356\pm0.132$ & $1.000\pm0.000$ &
$0.912\pm0.077$ \\
Copy $k=1$ & $0.975\pm0.014$ & $0.983\pm0.019$ & $0.435\pm0.069$ &
$0.850\pm0.251$ \\
Copy $k=3$ & $0.824\pm0.055$ & $0.949\pm0.032$ & $0.394\pm0.031$ &
$0.455\pm0.051$ \\
Copy $k=5$ & $0.455\pm0.033$ & $0.647\pm0.124$ & $0.325\pm0.013$ &
$0.389\pm0.012$ \\
Copy $k=10$ & $0.517\pm0.045$ & $0.465\pm0.009$ & $0.407\pm0.035$ &
$0.511\pm0.023$ \\
\bottomrule
\end{tabular}
\end{table}

\section{Discussion}
\label{sec:discussion}

\paragraph{What range captures.}
Temporal Range summarizes \emph{local, first-order} influence of the input
history on the final vector output via matrix-norm Jacobians. It answers “how
far back is the policy looking \emph{here}?” at a specific rollout point. In our
paired views, when the temporal profile concentrates near small lags
(Fig.~\ref{fig:prof-cartpole}), returns saturate with short windows
(Fig.~\ref{fig:abl-cartpole}); when profiles carry longer tails (e.g.,
Fig.~\ref{fig:prof-stateless}), performance recovers only once the window
exceeds the measured look-back (Fig.~\ref{fig:abl-stateless};
Table~\ref{tab:window_ablation_summary}).

\paragraph{Reading range with reward.}
Range is not reward. On near-Markov control, long tails (e.g., slow modes) can
be unnecessary yet harmless; under partial observability or noise, larger ranges
are often necessary but not sufficient. The alignment between ablations and
profiles (Figs.~\ref{fig:ablations}, \ref{fig:profiles}) and the aggregates
(Tables~\ref{tab:temporal_range_meanstd}, \ref{tab:performance_norm},
\ref{tab:window_ablation_summary}) provides a consistent behavioral cross-check:
short profiles go with early saturation, while longer-tailed profiles demand
larger windows before returns improve.

\paragraph{Practical guidance.}
Use $\hat{\rho}_T$ to audit memory use and choose the \emph{shortest sufficient}
context: (i) if return is high and $\hat{\rho}_T$ is large on a near-Markov
task, simplify the architecture or shorten context; (ii) if return is capped and
$\hat{\rho}_T$ is small on a partially observed task, increase memory capacity
or training horizon; (iii) if noise increases $\hat{\rho}_T$ without return
gains, you are smoothing more without extracting signal—revisit representation
or denoising.

\section{Conclusion}
We presented \emph{Temporal Range}, a first-order, model-agnostic measure of how
far back a trained policy effectively looks. By turning vector-output Jacobian
blocks into a temporal influence profile and summarizing by a magnitude-weighted
average lag, Temporal Range provides a single, interpretable number per sequence
and timestep. An axiomatic derivation for vector-output linear maps fixes the
form of both unnormalized and normalized variants and yields invariance to
uniform input and output rescaling. The metric is practical, taking one
reverse-mode pass per sequence, and, when the policy is non-differentiable or
black-box, a proxy LEM policy supplies reliable Jacobians.

Across POPGym diagnostics and control, Temporal Range is small in fully observed
control, scales with fixed task lags (Copy-$k$), and matches the minimum history
window required for near-optimal return as confirmed by window ablations. This
makes it useful for auditing memory dependence, comparing agents and
environments, and choosing the shortest sufficient context.

Limitations include locality (the measure is specific to the evaluated
rollout/time), dependence on preprocessing and norm choice, and the possibility
that slow modes inflate range without improving return. Future work includes
time-resolved profiles across decisions, per-output-component ranges (policy
vs.\ value), causal perturbation checks, and regularizers that penalize
unnecessary range to bias training toward simpler, shorter-context solutions.

\subsection*{Reproducibility Statement}
All code to reproduce our tables, figures, and ablations is available in an
anonymous repository:
\url{https://anonymous.4open.science/r/TemporalRange-26E4/README.md} . The repo
includes end-to-end training/evaluation scripts for PPO agents,
Jacobian/Temporal Range computation utilities, and the window-ablation driver.
Unless otherwise noted in the text, we use the hyperparameters listed in
Table~\ref{tab:hyperparams}. We provide fixed random seeds and configuration
files to regenerate the reported runs (curves use multiple seeds as indicated in
captions; summary tables average across episodes/trials as stated). The
repository also includes plotting code and exact evaluation commands to
reproduce
Tables~\ref{tab:temporal_range_meanstd}--\ref{tab:window_ablation_summary} and
Figs.~\ref{fig:ablations}--\ref{fig:noisy-all-profiles} from raw rollouts.

\subsection*{Acknowledgments} The research was sponsored in part by the Hector Foundation and by the Department of the
Air Force Artificial Intelligence Accelerator and was accomplished under
Cooperative Agreement Number FA8750-19-2-1000. The views and conclusions
contained in this document are those of the authors and should not be
interpreted as representing the official policies, either expressed or
implied, of the Department of the Air Force or the U.S. Government. The U.S.
Government is authorized to reproduce and distribute reprints for Government
purposes notwithstanding any copyright notation herein.

\bibliography{iclr2026_conference}

@inproceedings{morad_popgym_2023,
    title={{POPG}ym: Benchmarking Partially Observable Reinforcement Learning},
    author={Steven Morad and Ryan Kortvelesy and Matteo Bettini and Stephan Liwicki and Amanda Prorok},
    booktitle={The Eleventh International Conference on Learning Representations},
    year={2023},
}

@inproceedings{lu_structured_2023,
  title={Structured State Space Models for In-Context Reinforcement Learning},
  author={Lu, Chris and Schroecker, Yannick and Gu, Albert and Parisotto, Emilio and Foerster, Jakob and Singh, Satinder and Behbahani, Feryal},
  booktitle={37th Conference on Neural Information Processing Systems},
  year={2023}
}

@inproceedings{bamberger_measuring_2025,
    author = {Jacob Bamberger and Benjamin Gutteridge and Scott le Roux and Michael M. Bronstein and Xiaowen Dong},
    title = {On Measuring Long-Range Interactions in Graph Neural Networks},
    booktitle = {ICML},
    year = {2025}
}

@article{gu2023mamba,
  title={Mamba: Linear-time sequence modeling with selective state spaces},
  author={Gu, Albert and Dao, Tri},
  journal={arXiv preprint arXiv:2312.00752},
  year={2023}
}

@article{gu2021efficiently,
  title={Efficiently modeling long sequences with structured state spaces},
  author={Gu, Albert and Goel, Karan and R{\'e}, Christopher},
  journal={arXiv preprint arXiv:2111.00396},
  year={2021}
}

@inproceedings{hausknecht2015deep,
  title={Deep Recurrent Q-Learning for Partially Observable MDPs.},
  author={Hausknecht, Matthew J and Stone, Peter},
  booktitle={AAAI fall symposia},
  volume={45},
  pages={141},
  year={2015}
}

@article{chung2014empirical,
  title={Empirical evaluation of gated recurrent neural networks on sequence modeling},
  author={Chung, Junyoung and Gulcehre, Caglar and Cho, KyungHyun and Bengio, Yoshua},
  journal={arXiv preprint arXiv:1412.3555},
  year={2014}
}

@inproceedings{rusch2025linoss,
  title={Oscillatory State-Space Models},
  author={Rusch, T Konstantin and Rus, Daniela},
  booktitle={International Conference on Learning Representations},
  year={2025}
}

@inproceedings{pascanu2013difficulty,
  title={On the difficulty of training recurrent neural networks},
  author={Pascanu, Razvan and Mikolov, Tomas and Bengio, Yoshua},
  booktitle={International conference on machine learning},
  pages={1310--1318},
  year={2013},
  organization={Pmlr}
}

@article{berner2019dota,
  title={Dota 2 with large scale deep reinforcement learning},
  author={Berner, Christopher and Brockman, Greg and Chan, Brooke and Cheung, Vicki and Debiak, Przemysław and Dennison, Christy and Farhi, David and Fischer, Quirin and Hashme, Shariq and Hesse, Chris and others},
  journal={arXiv preprint arXiv:1912.06680},
  year={2019}
}

@conference{efroni_provable_2022,
	author={Efroni, Yonathan and Jin, Chi and Krishnamurthy, Akshay and Miryoosefi, Sobhan},
	title={Provable Reinforcement Learning with a Short-Term Memory},
  	booktitle={Proceedings of Machine Learning Research},
	year={2022},
	volume={162},
	pages={5832 – 5850},
}

@article{vaswani2017attention,
  title={Attention is all you need},
  author={Vaswani, Ashish and Shazeer, Noam and Parmar, Niki and Uszkoreit, Jakob and Jones, Llion and Gomez, Aidan N and Kaiser, {\L}ukasz and Polosukhin, Illia},
  journal={Advances in neural information processing systems},
  volume={30},
  year={2017}
}

@article{hochreiter1997long,
  title={Long short-term memory},
  author={Hochreiter, Sepp and Schmidhuber, J{\"u}rgen},
  journal={Neural computation},
  volume={9},
  number={8},
  pages={1735--1780},
  year={1997},
  publisher={MIT press}
}

@Article{mizutani_totally_2017,
    author={Mizutani, Eiji
    and Dreyfus, Stuart},
    title={Totally model-free actor-critic recurrent neural-network reinforcement learning in non-Markovian domains},
    journal={Annals of Operations Research},
    year={2017},
    month={Nov},
    day={01},
    volume={258},
    number={1},
    pages={107-131},
}

@article{malekzadeh_active_2024,
    author = {Malekzadeh, Parvin and Plataniotis, Konstantinos N.},
    title = {Active Inference and Reinforcement Learning: A Unified Inference on Continuous State and Action Spaces Under Partial Observability},
    journal = {Neural Computation},
    volume = {36},
    number = {10},
    pages = {2073-2135},
    year = {2024},
    month = {09},
}

@inproceedings{simonyan_deep_2014,
      title={Deep Inside Convolutional Networks: Visualising Image Classification Models and Saliency Maps}, 
      author={Karen Simonyan and Andrea Vedaldi and Andrew Zisserman},
      year={2014},
      booktitle={ICLR 2014 workshop}
}

@inproceedings{greydanus_visualizing_2018,
      title={Visualizing and Understanding Atari Agents}, 
      author={Sam Greydanus and Anurag Koul and Jonathan Dodge and Alan Fern},
      year={2018},
      booktitle={ICML}
}

@INPROCEEDINGS{meng_memory_2021,
  author={Meng, Lingheng and Gorbet, Rob and Kulić, Dana},
  booktitle={2021 IEEE/RSJ International Conference on Intelligent Robots and Systems (IROS)}, 
  title={Memory-based Deep Reinforcement Learning for POMDPs}, 
  year={2021},
}

@Inbook{williams_reinforcement_2009,
    author="Williams, John K.",
    title="Reinforcement Learning of Optimal Controls",
    booktitle="Artificial Intelligence Methods in the Environmental Sciences",
    year="2009",
    publisher="Springer Netherlands",
    address="Dordrecht",
    pages="297--327",
}

@inproceedings{rusch_lem_2022,
  title={Long Expressive Memory for Sequence Modeling},
  author={Rusch, T Konstantin and Mishra, Siddhartha and Erichson, N Benjamin and Mahoney, Michael W},
  booktitle={International Conference on Learning Representations},
  year={2022}
}

@INPROCEEDINGS{koffi_novel_2020,
  author={Koffi, Tresor Y. and Tao, Cai and Epalle, Thomas Martial and Mensa-Bonsu, Benjamin},
  booktitle={2020 International Conference on Internet of Things and Intelligent Applications (ITIA)}, 
  title={A Novel Reinforcement Learning Algorithm Based on Hierarchical Memory}, 
  year={2020},
}

@INPROCEEDINGS{fallasmoya_measuring_2021,
  author={Fallas-Moya, Fabian and Duncan, Jeremiah and Samuel, Tabitha and Sadovnik, Amir},
  booktitle={2021 XLVII Latin American Computing Conference (CLEI)}, 
  title={Measuring the Impact of Memory Replay in Training Pacman Agents using Reinforcement Learning}, 
  year={2021},
}

@article{omidshafiei_decentralized_2017,
  author = {Shayegan Omidshafiei and Ali–Akbar Agha–Mohammadi and Christopher Amato and Shih–Yuan Liu and Jonathan P How and John Vian},
  title ={Decentralized control of multi-robot partially observable Markov decision processes using belief space macro-actions},
  journal = {The International Journal of Robotics Research},
  volume = {36},
  number = {2},
  pages = {231-258},
  year = {2017},
}

@article{chen2021decision,
  title={Decision transformer: Reinforcement learning via sequence modeling},
  author={Chen, Lili and Lu, Kevin and Rajeswaran, Aravind and Lee, Kimin and Grover, Aditya and Laskin, Misha and Abbeel, Pieter and Srinivas, Aravind and Mordatch, Igor},
  journal={Advances in neural information processing systems},
  volume={34},
  pages={15084--15097},
  year={2021}
}

@misc{cherepanov_memory_2025,
  title={Memory, Benchmark and Robots: A Benchmark for Solving Complex Tasks with Reinforcement Learning}, 
  author={Egor Cherepanov and Nikita Kachaev and Alexey K. Kovalev and Aleksandr I. Panov},
  year={2025},
  eprint={2502.10550},
  archivePrefix={arXiv},
  primaryClass={cs.LG},
}

@inproceedings{zhang_robust_2020,
  author = {Zhang, Huan and Chen, Hongge and Xiao, Chaowei and Li, Bo and Liu, Mingyan and Boning, Duane and Hsieh, Cho-Jui},
  title = {Robust deep reinforcement learning against adversarial perturbations on state observations},
  year = {2020},
  booktitle = {Advances in Neural Information Processing Systems},
}

@misc{schulman_proximal_2017,
      title={Proximal Policy Optimization Algorithms}, 
      author={John Schulman and Filip Wolski and Prafulla Dhariwal and Alec Radford and Oleg Klimov},
      year={2017},
      eprint={1707.06347},
      archivePrefix={arXiv},
      primaryClass={cs.LG},
}

@article{kaelbling_planning_1998,
  author = {Kaelbling, Leslie Pack and Littman, Michael L. and Cassandra, Anthony R.},
  title = {Planning and Acting in Partially Observable Stochastic Domains},
  journal = {Artificial Intelligence},
  year = {1998},
  volume = {10},
  pages = {99--134},
}

@book{sutton_reinforcement_2018,
  author = {Sutton, Richard S. and Barto, Andrew G.},
  title = {Reinforcement Learning: An Introduction},
  edition = {2nd},
  publisher = {MIT Press},
  year = {2018},
}

@article{werbos_backpropagation_1990,
  author = {Werbos, Paul J.},
  title = {Backpropagation Through Time: What It Does and How to Do It},
  journal = {Proceedings of the IEEE},
  year = {1990},
  volume = {78},
  number = {10},
  pages = {1550--1560},
}

@inproceedings{khandelwal_sharp_2018,
  author = {Khandelwal, Urvashi and He, He and Qi, Peng and Jurafsky, Dan},
  title = {Sharp Nearby, Fuzzy Far Away: How Neural Language Models Use Context},
  booktitle = {Proceedings of ACL},
  year = {2018},
  pages = {284--294},
}

@inproceedings{adebayo_sanity_2018,
  author = {Adebayo, Julius and Gilmer, Justin and Muelly, Michael and Goodfellow, Ian and Hardt, Moritz and Kim, Been},
  title = {Sanity Checks for Saliency Maps},
  booktitle = {Advances in Neural Information Processing Systems (NeurIPS)},
  year = {2018},
}

@misc{tallec2018recurrentneuralnetworkswarp,
      title={Can recurrent neural networks warp time?}, 
      author={Corentin Tallec and Yann Ollivier},
      year={2018},
      booktitle={ICLR},
      eprint={1804.11188},
      archivePrefix={arXiv},
      primaryClass={cs.LG},
}

@misc{sundararajan2017axiomaticattributiondeepnetworks,
    title={Axiomatic Attribution for Deep Networks}, 
    author={Mukund Sundararajan and Ankur Taly and Qiqi Yan},
    year={2017},
    eprint={1703.01365},
    archivePrefix={arXiv},
    primaryClass={cs.LG},
    booktitle={ICML}
}

@misc{smilkov2017smoothgradremovingnoiseadding,
      title={SmoothGrad: removing noise by adding noise}, 
      author={Daniel Smilkov and Nikhil Thorat and Been Kim and Fernanda Viégas and Martin Wattenberg},
      year={2017},
      eprint={1706.03825},
      archivePrefix={arXiv},
}

@misc{koh2020understandingblackboxpredictionsinfluence,
      title={Understanding Black-box Predictions via Influence Functions}, 
      author={Pang Wei Koh and Percy Liang},
      year={2020},
      eprint={1703.04730},
      archivePrefix={arXiv},
      booktitle={ICML} 
}

@misc{pruthi2020estimatingtrainingdatainfluence,
      title={Estimating Training Data Influence by Tracing Gradient Descent}, 
      author={Garima Pruthi and Frederick Liu and Mukund Sundararajan and Satyen Kale},
      year={2020},
      eprint={2002.08484},
      archivePrefix={arXiv},
      primaryClass={cs.LG},
      booktitle={NeurIPS}
}

@software{jax2018github,
  author = {James Bradbury and Roy Frostig and Peter Hawkins and Matthew James Johnson and Chris Leary and Dougal Maclaurin and George Necula and Adam Paszke and Jake Vander{P}las and Skye Wanderman-{M}ilne and Qiao Zhang},
  title = {{JAX}: composable transformations of {P}ython+{N}um{P}y programs},
  url = {http://github.com/jax-ml/jax},
  version = {0.3.13},
  year = {2018},
}

@misc{schulman2016highdimensionalcontinuouscontrolusing,
      title={High-Dimensional Continuous Control Using Generalized Advantage Estimation}, 
      author={John Schulman and Philipp Moritz and Sergey Levine and Michael Jordan and Pieter Abbeel},
      year={2016},
      eprint={1506.02438},
      archivePrefix={arXiv},
      primaryClass={cs.LG},
      booktitle={ICLR} 
}
\bibliographystyle{iclr2026_conference}

\clearpage

\appendix
\section{Appendix}

\subsection{Window ablation results}
\label{app:ablations}

\begin{table}[h!]
\centering
\small
\caption{Window ablation summary (normalized to [0,1]). Best@m
shows peak performance and required window size; Avg. shows mean across all
windows.}
\label{tab:window_ablation_summary}
\begin{tabular*}{\textwidth}{@{\extracolsep{\fill}} l r r r r @{}}
\toprule
Environment & LEM & GRU & LSTM & LinOSS \\
\midrule
CartPole & 1.000@4 / 0.449 & 1.000@1 / 1.000 & 1.000@1 / 0.744 & 1.000@64 /
0.817 \\
Stateless CartPole & 0.829@8 / 0.327 & 0.961@8 / 0.335 & 0.501@16 / 0.215 &
1.000@16 / 0.512 \\
Noisy Stateless CartPole & 0.809@16 / 0.506 & 1.000@16 / 0.574 & 0.767@16 /
0.471 & 0.655@8 / 0.314 \\
RepeatFirst & 1.000@1 / 0.396 & 0.682@8 / 0.328 & 0.719@16 / 0.525 & 0.944@16 /
0.512 \\
Copy $k=1$ & 0.985@4 / 0.567 & 1.000@2 / 0.585 & 0.645@2 / 0.270 & 0.747@4 /
0.473 \\
Copy $k=3$ & 0.606@16 / 0.309 & 1.000@8 / 0.480 & 0.551@4 / 0.201 & 0.234@4 /
0.101 \\
Copy $k=5$ & 0.664@8 / 0.217 & 1.000@8 / 0.348 & 0.437@8 / 0.126 & 0.381@8 /
0.136 \\
Copy $k=10$ & 0.868@16 / 0.257 & 0.233@16 / 0.136 & 0.899@16 / 0.226 & 1.000@16
/ 0.276 \\
\bottomrule
\end{tabular*}
\end{table}

\subsection{Copy-$k$ calibration}

\begin{table}[h!]
\centering
\small
\setlength{\tabcolsep}{7pt}
\caption{Copy-$k$ calibration: MAE($\hat{\rho}_T$, $k$) over
episodes. Lower values indicate better alignment between measured TR and
ground-truth memory requirement $k$.}
\label{tab:copyk_mae}
\begin{tabular}{lcccc}
\toprule
Environment  & LEM & GRU & LSTM & LinOSS \\
\midrule
Copy $k=1$ & $9.65\pm0.42$ & $7.40\pm1.55$ & $11.29\pm0.80$ & $10.11\pm0.30$ \\
Copy $k=3$ & $9.13\pm0.51$ & $4.92\pm0.50$ & $14.31\pm0.88$ & $10.30\pm0.96$ \\
Copy $k=5$ & $10.14\pm0.59$ & $5.45\pm1.54$ & $12.25\pm1.18$ & $9.69\pm0.67$ \\
Copy $k=10$ & $6.63\pm0.58$ & $2.28\pm1.28$ & $5.91\pm3.38$ & $6.72\pm0.65$ \\
\bottomrule
\end{tabular}
\end{table}

\subsection{Deriving temporal range from axioms (vector-output linear maps)}
\label{app:axioms-uniqueness}

Inspired by recent axiomatic treatments of “range” in other domains
\citep{bamberger_measuring_2025}, we show that in the temporal setting a small
set of natural conditions uniquely fix both the unnormalized and normalized
forms. We first characterize the linear case to fix the form of any reasonable
``how-far-back'' score with vector outputs and matrix norms. Consider a
length-$T$ linear map
\begin{equation}
\label{eq:linear_map_vector}
L(z_1,\dots,z_T)\;=\;\sum_{t=1}^{T} B_t\,z_t\qquad(B_t\in\mathbb{R}^{c\times d}),
\end{equation}
where $z_t\in\mathbb{R}^d$ is the input at time $t$, and $\ell(t)=T{-}t$. We ask
for a score $\rho_T(L)\in(0,\infty)$ satisfying:

\begin{description}[leftmargin=*, itemsep=2pt, topsep=2pt]
\item[R1-u (single-step calibration with magnitude).] If $L(z)=B\,z_{T-k}$ with
any nonzero $B\in\mathbb{R}^{c\times d}$, then
$\rho_T(L)=\|B\|_{\text{mat}}\;k$.

\item[R2 (additivity over disjoint times).] If $L_1$ and $L_2$ depend on
disjoint time indices, then $\rho_T(L_1+L_2)=\rho_T(L_1)+\rho_T(L_2)$.

\item[R3 (absolute homogeneity).] For any $\alpha\in\mathbb{R}$, $\rho_T(\alpha
L)=|\alpha|\,\rho_T(L)$.
\end{description}

\begin{proposition}[Uniqueness of the unnormalized form for vector outputs]
\label{thm:unnormalized-unique-matrix}
Fix any matrix norm $\|\!\cdot\!\|_{\text{mat}}$ on $\mathbb{R}^{c\times d}$.
There is a unique nonnegative map $\rho_T$ on linear maps obeying
\textup{R1-u--R3}, namely
\begin{equation}
\label{eq:rho-unnorm-linear-matrix}
\rho_T(L)\;=\;\sum_{t=1}^{T} \|B_t\|_{\text{mat}}\,\ell(t).
\end{equation}
\end{proposition}

\begin{proof}[Proof of Proposition~\ref{thm:unnormalized-unique-matrix}]
Let $L_t$ denote the projection $L_t(z)=z_t$. Then $L=\sum_{t=1}^{T} B_t L_t$,
and $\{B_t L_t\}$ depend on disjoint time indices. By R2 and R3,
\[
\rho_T(L)=\sum_{t=1}^{T} \rho_T(B_t L_t)
=\sum_{t=1}^{T} \|B_t\|_{\mathrm{mat}}\,
\rho_T\!\Big(\tfrac{B_t}{\|B_t\|_{\mathrm{mat}}}\,L_t\Big).
\]
By single-step calibration (R1-u),
$\rho_T(\tfrac{B_t}{\|B_t\|_{\mathrm{mat}}}L_t)=\ell(t)$, yielding
$\rho_T(L)=\sum_{t=1}^{T}\|B_t\|_{\mathrm{mat}}\,\ell(t)$. Uniqueness follows
because R1-u fixes single-index values and R2--R3 propagate to disjoint sums.
\end{proof}

For an average lag, replace additivity with magnitude-weighted averaging.

\begin{description}[leftmargin=*, itemsep=2pt, topsep=2pt]
\item[R4 (magnitude-weighted averaging).] If $L_1$ and $L_2$ depend on disjoint
time indices and are nonzero, then for any $\alpha,\beta\in\mathbb{R}$ with
$(\alpha,\beta)\neq(0,0)$ and $\alpha L_1+\beta L_2\neq 0$,
\[
\hat{\rho}_T(\alpha L_1+\beta L_2)
\;=\;\frac{|\alpha|\,\hat{\rho}_T(L_1)+|\beta|\,\hat{\rho}_T(L_2)}{|\alpha|+|\beta|}.
\]

\item[R1-n (single-step calibration without magnitude).] If $L(z)=B\,z_{T-k}$
with any nonzero $B\in\mathbb{R}^{c\times d}$, then $\hat{\rho}_T(L)=k$.
\end{description}

\begin{proposition}[Uniqueness of the normalized form for vector outputs]
\label{thm:normalized-unique-matrix}
Fix a matrix norm $\|\!\cdot\!\|_{\text{mat}}$ on $\mathbb{R}^{c\times d}$. On
the domain of nonzero linear maps $L$ of the form \eqref{eq:linear_map_vector},
there is a unique map $\hat{\rho}_T$ obeying R1-n and R4, namely
\begin{equation}
\label{eq:rho-hat-linear-matrix}
\hat{\rho}_T(L)\;=\;
\frac{\sum_{t=1}^{T} \|B_t\|_{\text{mat}}\,\ell(t)}{\sum_{t=1}^{T} \|B_t\|_{\text{mat}}}.
\end{equation}
\end{proposition}

\begin{proof}[Proof of Proposition~\ref{thm:normalized-unique-matrix}]
By R4, for linear maps on disjoint time indices the normalized score must be a
magnitude-weighted average; with vector outputs the magnitudes are the matrix
norms $\|B_t\|_{\mathrm{mat}}$. Thus
\[
\hat{\rho}_T(L)=
\frac{\sum_{t=1}^{T}\|B_t\|_{\mathrm{mat}}\,\ell(t)}{\sum_{t=1}^{T}\|B_t\|_{\mathrm{mat}}}.
\]
On the domain of nonzero maps $L$, we have
$\sum_{t=1}^{T}\|B_t\|_{\mathrm{mat}}>0$, so the denominator is strictly
positive. R1-n calibrates single-step maps to $\ell(t)$, which uniquely fixes
the form.
\end{proof}

\paragraph{Takeaway.} On vector-output linear maps, the unnormalized sum of
matrix-norm–weighted lags and its normalized average are uniquely determined by
minimal, natural rules. This justifies using
\eqref{eq:unnormalized_range}--\eqref{eq:normalized_range} with
$w_t=\frac{1}{T-t}\sum_{s=t+1}^{T}\|J_{s,t}\|_{\text{mat}}$ as
canonical summaries of temporal influence for nonlinear policies via local
linearization.

\subsection{Analytical calibration of Temporal Range}
\label{sec:analyticalTR}
We compute Temporal Range exactly in two differentiable settings where all
Jacobians are available in closed form, now with \emph{vector} outputs and
matrix norms.

\subsubsection{Exact range in Copy-$k$}
\label{sec:copyk-analytic}
Fix a sequence length $T$ and inputs $X_{1:T} = [x_1,\ldots,x_T]$,
$x_t\in\mathbb{R}^d$. Consider the map $F:\mathbb{R}^{T\times d}\to\mathbb{R}^c$
that \emph{copies} a linear readout of the observation from $k$ steps ago at the
final time:
\[
y_T(X) \;=\; U\,x_{T-k}, \qquad U\in\mathbb{R}^{c\times d},\quad k\in\{0,\ldots,T{-}1\}.
\]
In the multi-output formulation, outputs at earlier steps $s<T$
are zero, so $J_{s,t}(X)=0$ for all $s<T$ and all $t$. The only nonzero Jacobian
is the Jacobian block of $y_T$ with respect to $x_t$:
\[
J_{T,t}(X) \;=\; \frac{\partial y_T}{\partial x_t}(X) \;=\;
\begin{cases}
U, & t = T-k,\\
0, & \text{otherwise.}
\end{cases}
\]
Thus averaging over $s\in\{t+1,\ldots,T\}$ yields
$w_t(X)=\frac{1}{T-t}\|U\|_{\text{mat}}\,\mathbf{1}\{t=T-k\}$ for $t<T$; the
normalization factor $(T-t)^{-1}$ cancels in $\hat{\rho}_T$. With the lag
$\ell(t)=T-t$, the unnormalized and normalized ranges are
\[
\rho_T(F;X) \;=\; \sum_{t=1}^T w_t\,\ell(t) \;=\; \|U\|_{\text{mat}}\cdot k,
\qquad
\hat\rho_T(F;X) \;=\; \frac{\sum_t w_t\,\ell(t)}{\sum_t w_t}
\;=\; k.
\]
Thus \textbf{$\hat\rho_T=k$ exactly}, matching the ground-truth offset.

\paragraph{Invariances in this setting.}
If we apply uniform output rescaling $\tilde y_T=\alpha y_T$, then all $w_t$
scale by $|\alpha|$, leaving $\hat\rho_T$ unchanged. If we uniformly rescale
inputs $x_t^\star=\beta x_t$ with $\beta\neq 0$, then $\|{\partial
y_T}/{\partial x_t^\star}\|_{\text{mat}} =\|{\partial y_T}/{\partial
x_t}\|_{\text{mat}}/|\beta|$ for all $t$, again leaving $\hat\rho_T$ unchanged
(while $\rho_T$ rescales).

\subsubsection{Exact range under linear recurrent readout}
\label{sec:linear-analytic}
Consider a linear recurrent “memory”
\[
h_{t+1} \;=\; A h_t + C x_{t+1},\qquad h_0=0,
\]
and define the vector output at the final time by a linear readout $y_T \;=\;
Q\,h_T$ with $Q\in\mathbb{R}^{c\times p}$. (For simplicity we
present the single final-output case; if outputs are emitted at each step $y_s=Q
h_s$, then $w_t$ becomes the average of $\|Q A^{s-t-1} C\|_{\text{mat}}$ over
$s\in\{t+1,\ldots,T\}$, exhibiting similar exponential decay.) Unrolling,
\[
h_T \;=\; \sum_{t=1}^{T} A^{T-t} C\,x_t
\quad\Longrightarrow\quad
\frac{\partial y_T}{\partial x_t} \;=\; Q\,A^{T-t} C\;\in\;\mathbb{R}^{c\times d}.
\]
For any matrix norm $\|\!\cdot\!\|_{\text{mat}}$ on $\mathbb{R}^{c\times d}$,
the per-step weight is
\[
w_t \;=\; \big\| Q\,A^{T-t} C \big\|_{\text{mat}}
\;=\; \big\| Q\,A^{\ell(t)} C \big\|_{\text{mat}}.
\]
Hence the influence profile $t\mapsto w_t$ is governed by powers of $A$.

\paragraph{Takeaway.}
In a linear recurrent readout, $w_t$ is determined by the propagator powers
$A^{\ell}$, and $\hat\rho_T$ becomes the expected lag under those induced
weights. This bridges spectral properties of the memory dynamics and the
measured temporal range.

\subsubsection{From exact settings to complex simulators}
The two cases above show that (i) when dependence is concentrated at a single
offset (\textsc{Copy-$k$}), the normalized range recovers $k$ exactly; and (ii)
when dependence is distributed and exponentially decaying (linear recurrence),
$\hat\rho_T$ has a closed form that grows with the effective memory timescale.
In complex environments, the same definitions apply; we compute Jacobians with
respect to \emph{observations} by differentiating the policy. When the policy is
not amenable to automatic differentiation or is only available as a black box,
we compute the same $w_t$ via reverse mode on a compact LEM proxy (see
\S\ref{sec:lem}); in practice the resulting $\hat{\rho}_T$ tracks ground-truth
diagnostics and performance-relevant memory in control tasks.

\subsection{Use of large language models}
We used large language model (LLM) assistants for writing support: reorganizing
sentences and paragraphs for clarity, tightening prose, fixing grammar and LaTeX
formatting, and suggesting alternative phrasings.

\subsection{Comparison of aggregation operators}
\label{app:max-aggregation}

We investigated using $\bigoplus = \max$ instead of $\bigoplus = \text{mean}$ for the influence weight computation. Table~\ref{tab:temporal_range_max} shows TR values with max aggregation.

\begin{table}[H]
\centering
\small
\setlength{\tabcolsep}{5pt}
\caption{Temporal Range $\hat{\rho}_T$ using $\bigoplus = \max$ aggregation (steps; mean $\pm$ std over episodes). Values cluster around 15--18 for all tasks, reducing discriminative power compared to mean aggregation (Table~\ref{tab:temporal_range_meanstd}).}
\label{tab:temporal_range_max}
\begin{tabular}{lcccc}
\toprule
Environment  & LEM & GRU & LSTM & LinOSS \\
\midrule
CartPole & $15.97\pm0.24$ & $16.12\pm0.21$ & $16.37\pm0.57$ & $16.07\pm0.32$ \\
Stateless CartPole & $15.96\pm0.39$ & $16.14\pm0.44$ & $17.11\pm1.66$ & $15.99\pm0.78$ \\
Noisy Stateless CartPole & $15.71\pm0.59$ & $15.67\pm0.59$ & $17.78\pm1.48$ & $15.91\pm0.44$ \\
RepeatFirst & $19.52\pm1.79$ & $16.88\pm1.01$ & $16.72\pm1.56$ & $21.85\pm0.98$ \\
Copy $k=1$ & $16.00\pm0.22$ & $15.96\pm0.25$ & $15.47\pm0.31$ & $15.46\pm0.29$ \\
Copy $k=3$ & $15.94\pm0.20$ & $15.94\pm0.22$ & $18.94\pm1.19$ & $15.37\pm0.45$ \\
Copy $k=5$ & $16.70\pm0.43$ & $16.03\pm0.26$ & $18.82\pm1.40$ & $15.77\pm0.49$ \\
Copy $k=10$ & $17.93\pm0.71$ & $17.35\pm0.88$ & $18.06\pm2.62$ & $17.42\pm0.83$ \\
\bottomrule
\end{tabular}
\end{table}

Using max causes TR values to cluster, reducing discrimination between tasks. For example, Copy $k=1$ yields TR $\approx$ 15--16, nearly identical to Copy $k=10$ at $\approx$ 17--18. With mean aggregation (Table~\ref{tab:temporal_range_meanstd}), Copy $k=1$ produces TR $\approx$ 8--12 while Copy $k=10$ produces TR $\approx$ 12--17, correctly reflecting increased memory demands. The max operator is dominated by the single largest Jacobian norm, which tends to be similar across tasks. Mean provides a more robust measure of \emph{sustained} temporal influence.

\subsection{Additional temporal influence profiles}
\label{app:more-figs}

\begin{figure}[H]
    \centering
    \begin{subfigure}[b]{0.48\textwidth}
        \includegraphics[width=\textwidth]{plots/RepeatPrevious3_lem_temporal_profile.png}
        \caption{LEM, Copy $k=3$}
    \end{subfigure}\hfill
    \begin{subfigure}[b]{0.48\textwidth}
        \includegraphics[width=\textwidth]{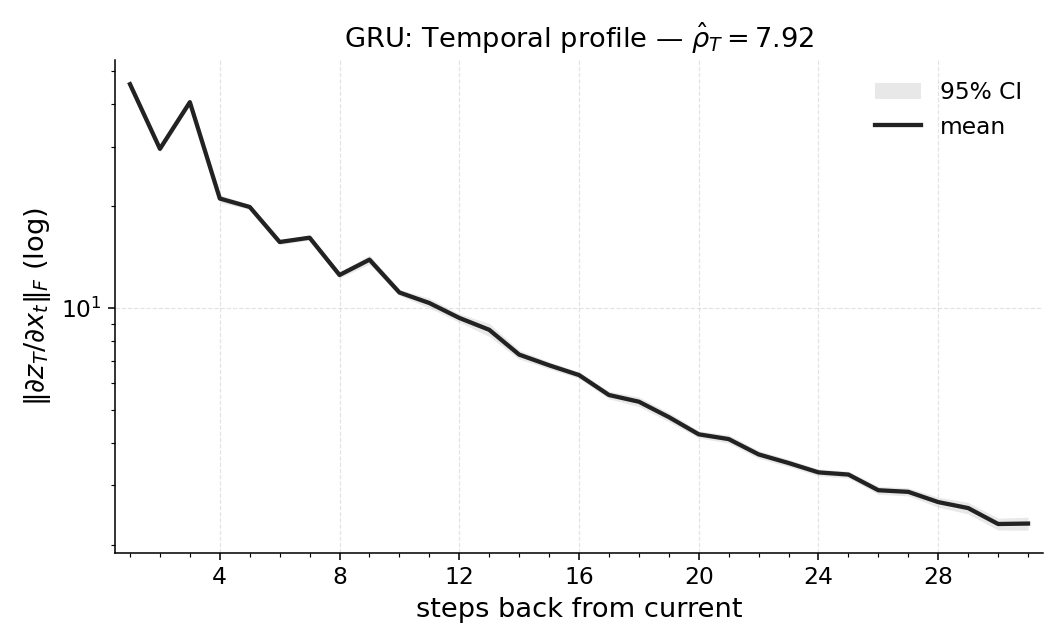}
        \caption{GRU, Copy $k=3$}
    \end{subfigure}

    \medskip

    \begin{subfigure}[b]{0.48\textwidth}
        \includegraphics[width=\textwidth]{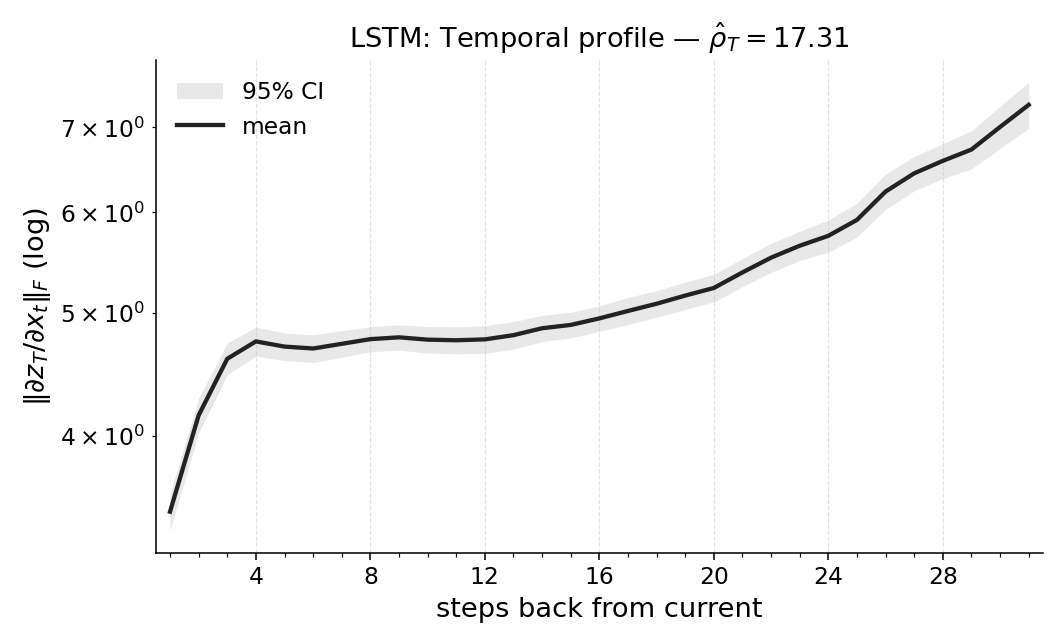}
        \caption{LSTM, Copy $k=3$}
    \end{subfigure}\hfill
    \begin{subfigure}[b]{0.48\textwidth}
        \includegraphics[width=\textwidth]{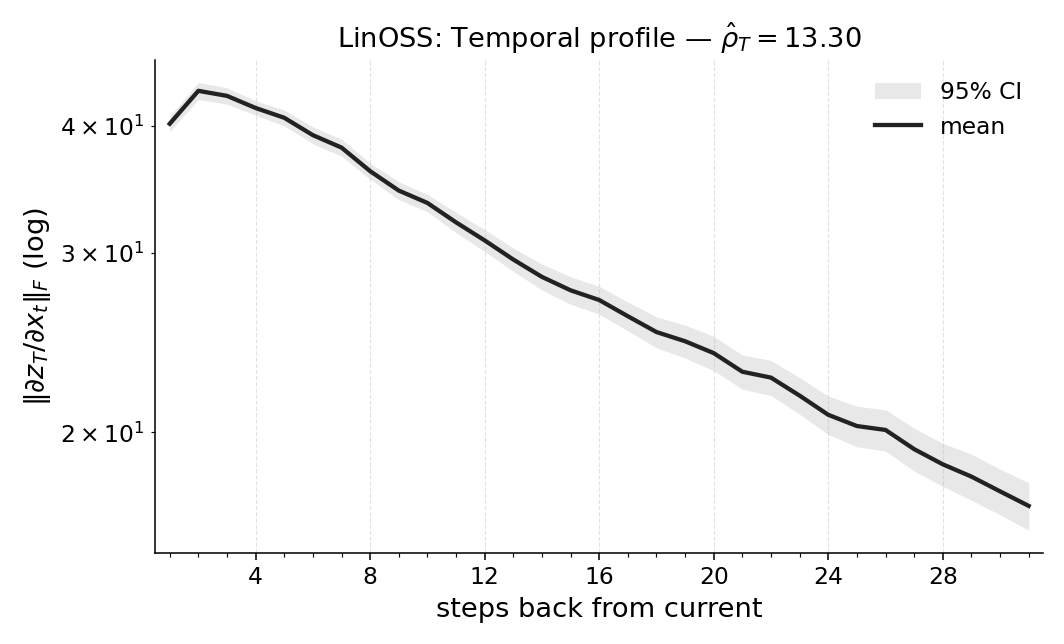}
        \caption{LinOSS, Copy $k=3$}
    \end{subfigure}
    \caption{Temporal influence profiles for Copy $k=3$ across
    all architectures (LEM, GRU, LSTM, LinOSS).}
    \label{fig:copyk3-all-profiles}
\end{figure}

\begin{figure}[H]
    \centering
    \begin{subfigure}[b]{0.48\textwidth}
        \includegraphics[width=\textwidth]{plots/RepeatPrevious10_lem_temporal_profile.png}
        \caption{LEM, Copy $k=10$}
    \end{subfigure}\hfill
    \begin{subfigure}[b]{0.48\textwidth}
        \includegraphics[width=\textwidth]{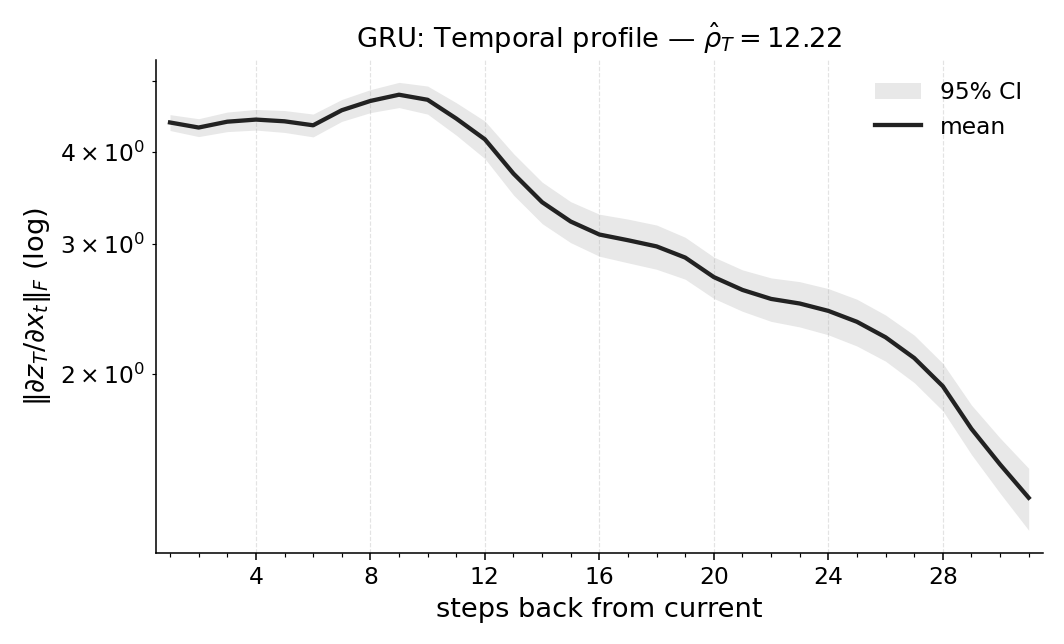}
        \caption{GRU, Copy $k=10$}
    \end{subfigure}

    \medskip

    \begin{subfigure}[b]{0.48\textwidth}
        \includegraphics[width=\textwidth]{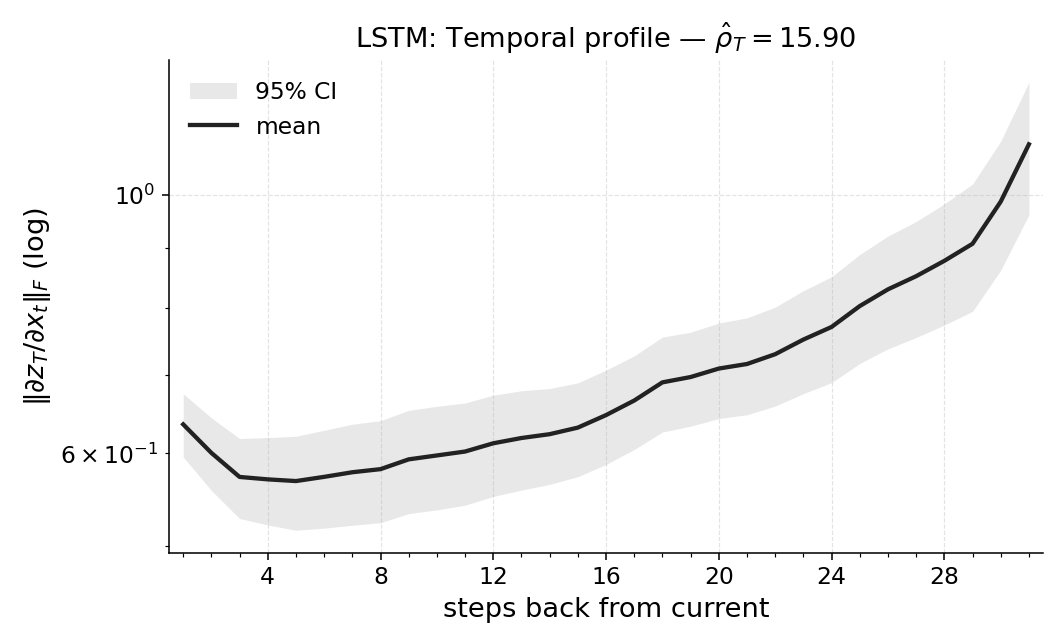}
        \caption{LSTM, Copy $k=10$}
    \end{subfigure}\hfill
    \begin{subfigure}[b]{0.48\textwidth}
        \includegraphics[width=\textwidth]{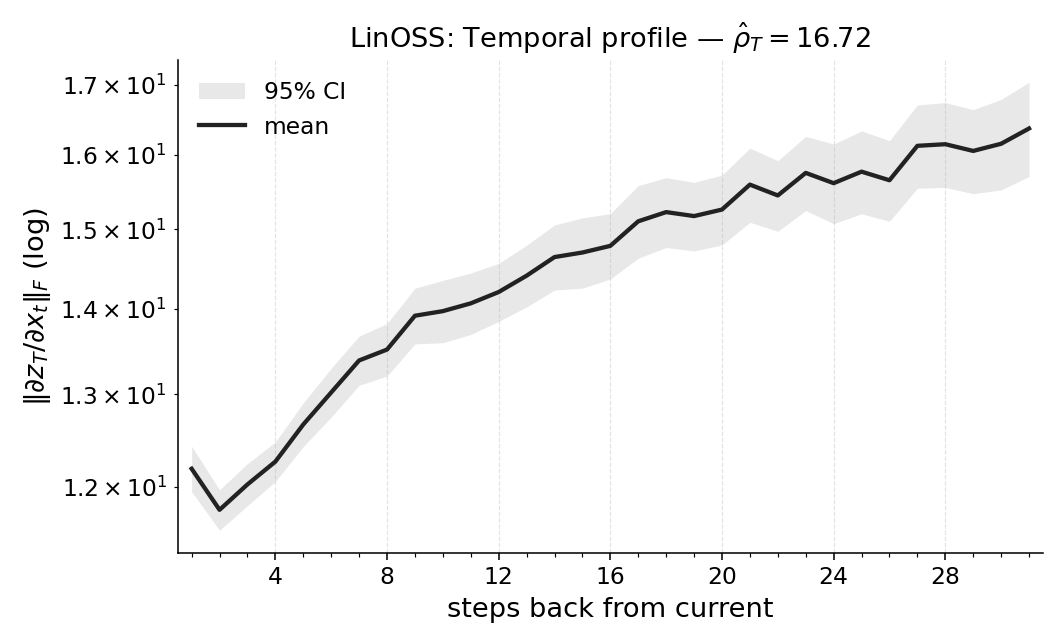}
        \caption{LinOSS, Copy $k=10$}
    \end{subfigure}
    \caption{Temporal influence profiles for Copy $k=10$ across
    all architectures (LEM, GRU, LSTM, LinOSS).}
    \label{fig:copyk10-all-profiles}
\end{figure}

\begin{figure}[H]
    \centering
    \begin{subfigure}[b]{0.48\textwidth}
        \includegraphics[width=\textwidth]{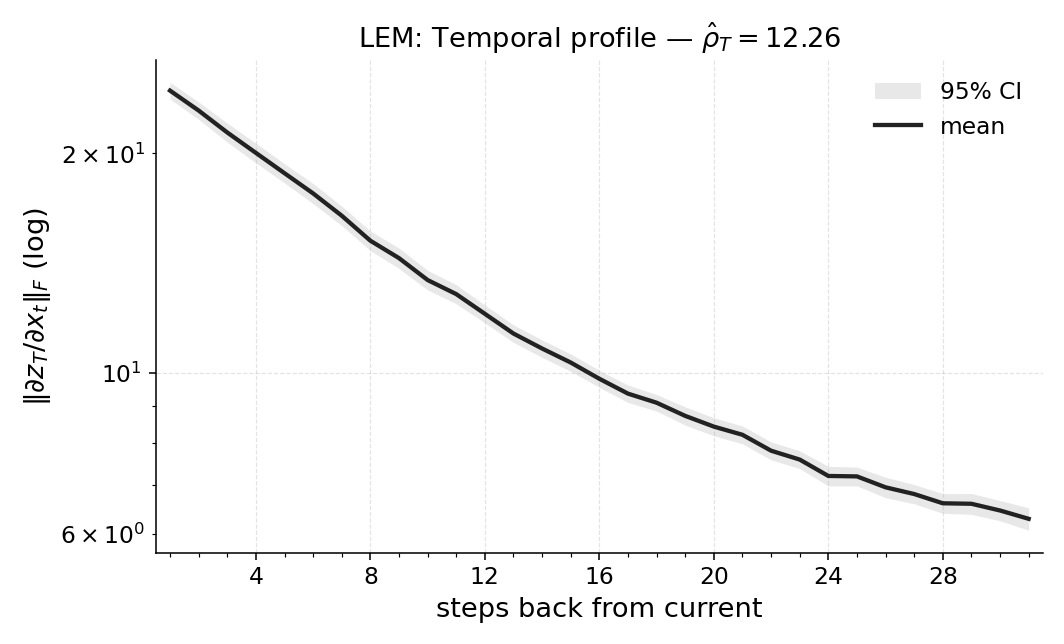}
        \caption{LEM, Noisy Stateless CartPole}
    \end{subfigure}\hfill
    \begin{subfigure}[b]{0.48\textwidth}
        \includegraphics[width=\textwidth]{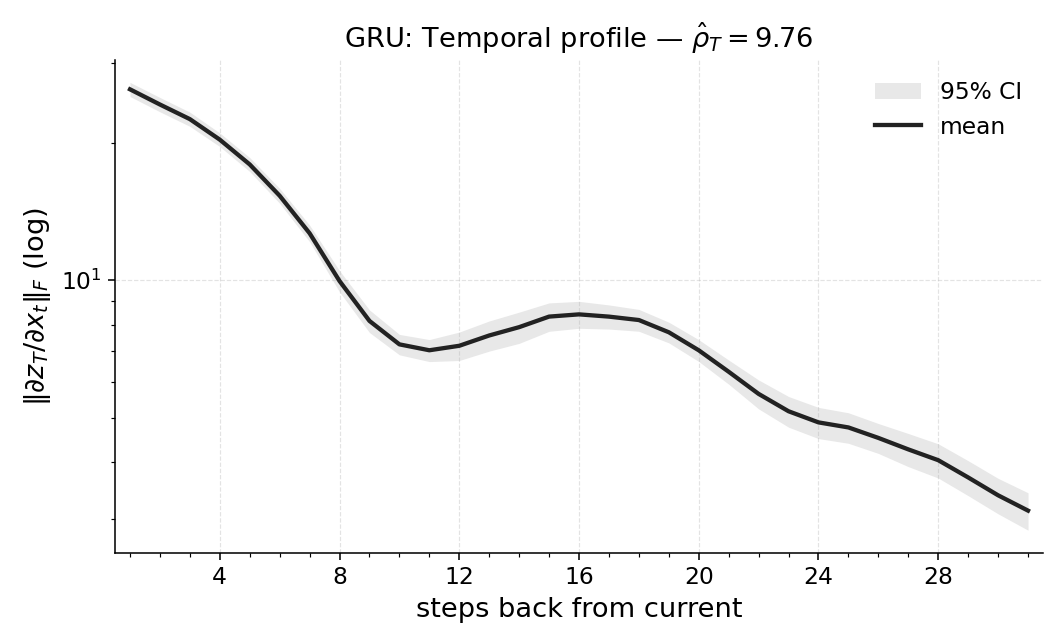}
        \caption{GRU, Noisy Stateless CartPole}
    \end{subfigure}

    \medskip

    \begin{subfigure}[b]{0.48\textwidth}
        \includegraphics[width=\textwidth]{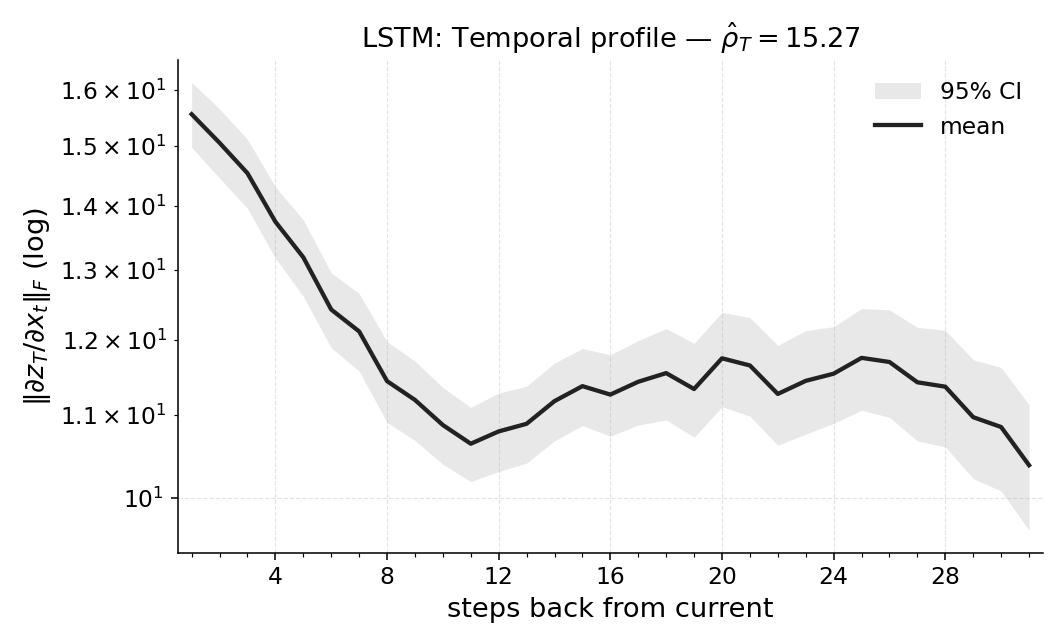}
        \caption{LSTM, Noisy Stateless CartPole}
    \end{subfigure}\hfill
    \begin{subfigure}[b]{0.48\textwidth}
        \includegraphics[width=\textwidth]{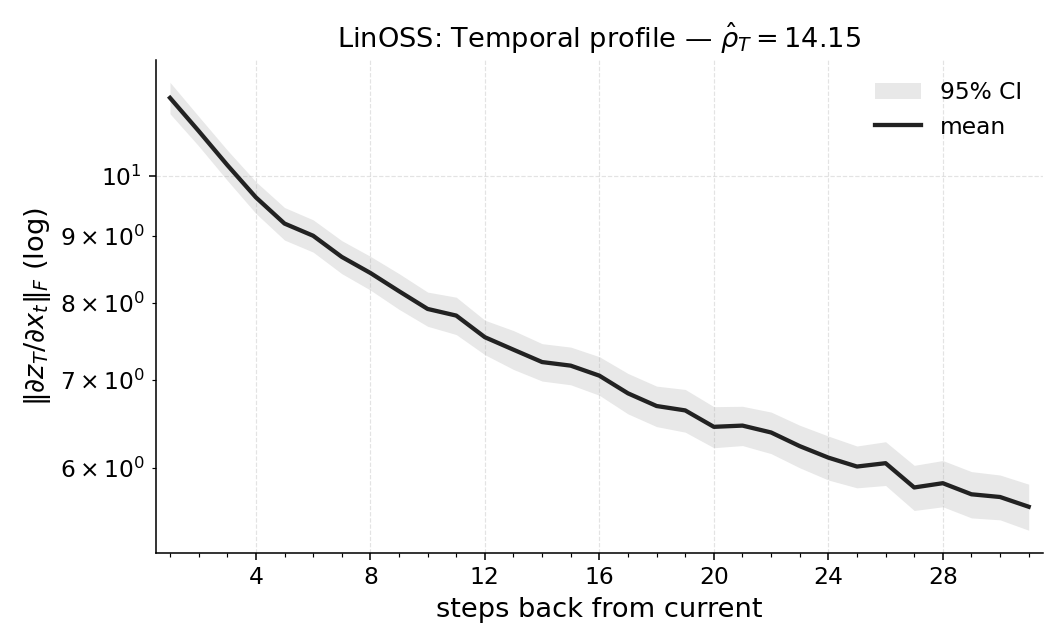}
        \caption{LinOSS, Noisy Stateless CartPole}
    \end{subfigure}
    \caption{Temporal influence profiles for Noisy Stateless
    CartPole across all architectures (LEM, GRU, LSTM, LinOSS).}
    \label{fig:noisy-all-profiles}
\end{figure}

\subsection{Hyperparameters}
\label{app:hyperparams}

See Table \ref{tab:hyperparams}.

\begin{table}[h]
\centering
\small
\caption{Training and evaluation hyperparameters used in all experiments.}
\label{tab:hyperparams}
\begin{tabular}{ll}
\toprule
\textbf{Hyperparameter} & \textbf{Value} \\
\midrule
\texttt{LR} & \texttt{3e-4} \\
\texttt{KL\_COEF} & \texttt{0.0} \\
\texttt{KL\_TARGET} & \texttt{0.01} \\
\texttt{NUM\_ENVS} & \texttt{64} \\
\texttt{NUM\_STEPS} & \texttt{256} \\
\texttt{TOTAL\_TIMESTEPS} & \texttt{1e7} \\
\texttt{UPDATE\_EPOCHS} & \texttt{8} \\
\texttt{NUM\_MINIBATCHES} & \texttt{4} \\
\texttt{GAMMA} & \texttt{0.99} \\
\texttt{GAE\_LAMBDA} & \texttt{0.95} \\
\texttt{CLIP\_EPS} & \texttt{0.2} \\
\texttt{ENT\_COEF} & \texttt{0.01} \\
\texttt{VF\_COEF} & \texttt{0.5} \\
\texttt{MAX\_GRAD\_NORM} & \texttt{0.5} \\
\texttt{HIDDEN\_SIZE} & \texttt{128} \\
\texttt{DENSE\_SIZE} & \texttt{128} \\
\texttt{GRU\_HIDDEN\_SIZE} & \texttt{128} \\
\texttt{LEM\_DT} & \texttt{0.5} \\
\texttt{ANNEAL\_LR} & \texttt{True} \\
\texttt{NUM\_TRIALS} & \texttt{3} \\
\texttt{T\_JACOBIAN} & \texttt{32} \\
\bottomrule
\end{tabular}
\end{table}

\end{document}